\documentclass[journal, oneclomn]{IEEEtran}

\ifCLASSOPTIONpeerreview
\begin{center} \bfseries EDICS: CIF-OBI, CIF-SBI, IMT-SIM\end{center}
\fi

\usepackage{xr-hyper}
\usepackage[switch]{lineno}

\usepackage{graphics,amssymb,amsmath,epsfig,color}
\usepackage{graphicx}
\usepackage{multibib}

\usepackage[hyphens]{url}     
\usepackage[hidelinks]{hyperref}       
\usepackage{booktabs}       
\usepackage{amsfonts}       
\usepackage{amssymb}        
\usepackage{nicefrac}       
\usepackage{microtype}      
\usepackage{bm}				
\usepackage{cite}			
\usepackage{graphicx}		
\usepackage{mathtools}		
\usepackage{algorithm}
\usepackage{algpseudocode}
\usepackage{amsthm}			
\usepackage{enumitem}		
\usepackage{multirow}       
\usepackage{tabularx}	    
\usepackage{hhline}
\usepackage{arydshln}		
\usepackage{xcolor}

\definecolor{lightgreen}{rgb}{.9,1,.9}

\newcolumntype{L}[1]{>{\raggedright\arraybackslash}p{#1}}
\newcolumntype{C}[1]{>{\centering\arraybackslash}p{#1}}
\newcolumntype{R}[1]{>{\raggedleft\arraybackslash}p{#1}}

\theoremstyle{plain} 
\newtheorem{proposition}{Proposition}
\newtheorem{definition}{Definition}
\newtheorem{theorem}{Theorem}
\newtheorem{lemma}{Lemma}

\newtheorem{assumption}{Assumption}

\def\defn{\,\coloneqq\,}
\def\argmin{\mathop{\mathsf{arg\,min}}} 

\def\lim{\mathop{\mathsf{lim}}} 
\def\min{\mathop{\mathsf{min}}}
\def\max{\mathop{\mathsf{max}}}
\def\dom{\mathop{\mathsf{dom}}}
\def\ri{\mathop{\mathsf{ri}}}

\def\prox{\mathsf{prox}}

\def\log{\mathsf{log}}
\def\zer{\mathsf{zer}}
\def\fix{\mathsf{fix}}

\def\proposed{IPA}

\def\Ibf{{\mathbf{I}}}

\def\ebm{{\bm{e}}}
\def\hbm{{\bm{h}}}
\def\sbm{{\bm{s}}}

\def\xbm{{\bm{x}}}

\def\gbm{{\bm{g}}}
\def\ybm{{\bm{y}}}
\def\zbm{{\bm{z}}}

\def\nbm{{\bm{n}}}

\def\vbm{{\bm{v}}}

\def\zerobm{\bm{0}}
\def\Abm{{\bm{A}}}
\def\Dbm{{\bm{D}}}

\def\vbmast{{\bm{v}^\ast}}
\def\xbmast{{\bm{x}^\ast}}

\def\zbmast{{\bm{z}^\ast}}
\def\sbmast{{\bm{s}^\ast}}

\def\xbmhat{{\widehat{\bm{x}}}}

\def\Tsf{{\mathsf{T}}}
\def\Ssf{{\mathsf{S}}}
\def\Dsf{{\mathsf{D}}}
\def\Gsf{{\mathsf{F}}}

\def\Nsf{{\mathsf{N}}}
\def\Gsf{{\mathsf{G}}}

\def\Gsfhat{{\widehat{\mathsf{G}}}}
\def\Isf{{\mathsf{I}}}
\def\Psf{{\mathsf{P}}}

\def\Rsf{{\mathsf{R}}}

\def\R{\mathbb{R}}

\def\E{\mathbb{E}}

\def\Ncal{{\mathcal{N}}}

\makeatletter
\renewcommand\subsubsection{\@startsection{subsubsection}{3}{\z@}%
        {3.5ex plus 1.5ex minus 1.5ex}%
        {0.7ex plus .5ex minus 0ex}%
        {\normalfont\normalsize\itshape}}
\makeatother

\begin{document}
\title{Scalable Plug-and-Play ADMM\\ with Convergence Guarantees}
\author{Yu~Sun$^*$,~\IEEEmembership{Student~Member,~IEEE},
                        Zihui~Wu$^*$,
                        Xiaojian~Xu$^*$,~\IEEEmembership{Student~Member,~IEEE},\\ Brendt~Wohlberg,~\IEEEmembership{Senior Member,~IEEE}, and Ulugbek~S.~Kamilov,~\IEEEmembership{Senior Member,~IEEE}%
\thanks{This material is based upon work supported by NSF award CCF-1813910 and by the Laboratory Directed Research and Development program of Los Alamos National Laboratory under project number 20200061DR. \emph{(Corresponding author: Ulugbek~S.~Kamilov.)}}
\thanks{Y.~Sun and X.~Xu is with the Department of
Computer Science \& Enginnering, Washington University in St.~Louis, MO 63130, USA.}
\thanks{Z..~Wu is with the Department of
Computer Science, California Institute of Technology, CA 91125, USA.}
\thanks{B.~Wohlberg is with Theoretical Division, Los Alamos National Laboratory, Los Alamos, NM 87545 USA.}
\thanks{U.~S.~Kamilov (email:~kamilov@wustl.edu)
is with the Department of
Computer Science \& Engineering and the Department of Electrical \& Systems Engineering, Washington University in St.~Louis, MO 63130, USA.}
\thanks{$^*$These authors contributed equally, and are listed in alphabetical order of their family names.}
}


\maketitle

\begin{abstract}
Plug-and-play priors (PnP) is a broadly applicable methodology for solving inverse problems by exploiting statistical priors specified as denoisers. Recent work has reported the state-of-the-art performance of PnP algorithms using pre-trained deep neural nets as denoisers in a number of imaging applications. However, current PnP algorithms are impractical in large-scale settings due to their heavy computational and memory requirements. This work addresses this issue by proposing an \emph{incremental} variant of the widely used PnP-ADMM algorithm, making it scalable to large-scale datasets. We theoretically analyze the convergence of the algorithm under a set of explicit assumptions, extending recent theoretical results in the area. Additionally, we show the effectiveness of our algorithm with nonsmooth data-fidelity terms and deep neural net priors, its fast convergence compared to existing PnP algorithms, and its scalability in terms of speed and memory.
\end{abstract}


\section{Introduction}

Plug-and-play priors (PnP) is a simple yet flexible methodology for imposing statistical priors without explicitly forming an objective function~\cite{Venkatakrishnan.etal2013, Sreehari.etal2016}. PnP algorithms alternate between imposing data consistency by minimizing a data-fidelity term and imposing a statistical prior by applying an additive white Gaussian noise (AWGN) denoiser. PnP draws its inspiration from the \emph{proximal algorithms} extensively used in nonsmooth composite optimization~\cite{Parikh.Boyd2014}, such as the proximal-gradient method (PGM)~\cite{Figueiredo.Nowak2003, Daubechies.etal2004, Bect.etal2004, Beck.Teboulle2009} and alternating direction method of multipliers (ADMM)~\cite{Eckstein.Bertsekas1992, Afonso.etal2010, Ng.etal2010, Boyd.etal2011}. The popularity of deep learning has led to a wide adoption of PnP for exploiting \emph{learned} priors specified through pre-trained deep neural nets, leading to its state-of-the-art performance in a variety of applications~\cite{Zhang.etal2017a, Dong.etal2019, Zhang.etal2019, Ahmad.etal2020, Wei.etal2020}. Its empirical success has spurred a follow-up work that provided theoretical justifications to PnP in various settings~\cite{Chan.etal2016, Meinhardt.etal2017, Buzzard.etal2017, Sun.etal2018a, Tirer.Giryes2019, Teodoro.etal2019, Ryu.etal2019}. Despite this progress, current PnP algorithms are not practical for addressing large-scale problems due to their computation time and memory requirements. To the best of our knowledge, the only prior work on developing PnP algorithms that are suitable for large-scale problems is the \emph{stochastic gradient descent variant of PnP (PnP-SGD)}, whose fixed-point convergence was recently analyzed for smooth data-fidelity terms~\cite{Sun.etal2018a}.

In this work, we present a new \emph{incremental PnP-ADMM (IPA)} algorithm for solving large-scale inverse problems.  As an extensions of the widely used PnP-ADMM~\cite{Venkatakrishnan.etal2013, Sreehari.etal2016}, IPA can integrate statistical information from a data-fidelity term and a pre-trained deep neural net. However, unlike PnP-ADMM, IPA can effectively scale to datasets that are too large for traditional batch processing by using a single element or a small subset of the dataset at a time. The memory and per-iteration complexity of IPA is independent of the number of measurements, thus allowing it to deal with very large datasets. Additionally, unlike PnP-SGD~\cite{Sun.etal2018a}, IPA can effectively address problems with \emph{nonsmooth} data-fidelity terms, and generally has faster convergence. We present a detailed convergence analysis of IPA under a set of explicit assumptions on the data-fidelity term and the denoiser. Our analysis extends the recent fixed-point analysis of PnP-ADMM in~\cite{Ryu.etal2019} to partial randomized processing of data. To the best of our knowledge, the proposed scalable PnP algorithm and corresponding convergence analysis are absent from the current literature in this area. Our numerical validation demonstrates the practical effectiveness of IPA for integrating nonsmooth data-fidelity terms and deep neural net priors, its fast convergence compared to PnP-SGD, and its scalability in terms of both speed and memory. In summary, we establish IPA as a flexible, scalable, and theoretically sound PnP algorithm applicable to a wide variety of large-scale problems.

\section{Background}
\label{Sec:Background}

Consider the problem of estimating an unknown vector $\xbm \in \R^n$ from a set of noisy measurements $\ybm \in \R^m$. It is standard practice to formulate the solution as an optimization problem
\begin{equation}
\label{Eq:RegularizedOptimization}
\min_{\xbm \in \R^n} f(\xbm) \quad\text{with}\quad f(\xbm) \defn g(\xbm) + h(\xbm),
\end{equation}
where $g$ is a data-fidelity term that quantifies consistency with the observed data $\ybm$ and $h$ is a regularizer that encodes prior knowledge on $\xbm$. As an example, consider the nonsmooth $\ell_1$-norm data-fidelity term $g(\xbm) = \|\ybm-\Abm\xbm\|_1$, which assumes a linear observation model $\ybm = \Abm\xbm + \ebm$, and the TV regularizer $h(\xbm)=\tau\|\Dbm\xbm\|_1$, where $\Dbm$ is the gradient operator and $\tau > 0$ is the regularization parameter. Common applications of~\eqref{Eq:RegularizedOptimization} include sparse vector recovery in compressive sensing~\cite{Candes.etal2006, Donoho2006}, image restoration using total variation (TV)~\cite{Beck.Teboulle2009a}, and low-rank matrix completion~\cite{Recht.etal2010}.

Proximal algorithms are often used for solving problems of form~\eqref{Eq:RegularizedOptimization} when $g$ or $h$ are nonsmooth~\cite{Parikh.Boyd2014}. For example, one such standard algorithm, ADMM, can be summarized as
\begin{subequations}
\begin{align}
&\zbm^k = \prox_{\gamma g}(\xbm^{k-1}+\sbm^{k-1})\\
&\xbm^k = \prox_{\gamma h}(\zbm^k - \sbm^{k-1})\\
&\sbm^k = \sbm^{k-1} + \xbm^k - \zbm^k,
\end{align}
\end{subequations}
where $\gamma > 0$ is the penalty parameter~\cite{Boyd.etal2011} and \emph{proximal operator} is defined as
\begin{equation}
\label{Eq:ProximalOperator}
\prox_{\tau h}(\zbm) \defn \argmin_{\xbm \in \R^n} \left\{\frac{1}{2}\|\xbm-\zbm\|_2^2 + \tau h(\xbm)\right\}
\end{equation}
for any proper, closed, and convex function $h$~\cite{Parikh.Boyd2014}. The proximal operator can be interpreted as a \emph{maximum a posteriori probability (MAP)} estimator for the AWGN denoising problem
\begin{equation}
\zbm = \xbm + \nbm \quad\text{where}\quad \xbm \sim p_\xbm, \quad \nbm \sim \Ncal(\zerobm, \tau \Ibf),
\end{equation}
by setting $h(\xbm) = -\log(p_\xbm(\xbm))$. This perspective inspired the development of PnP~\cite{Venkatakrishnan.etal2013, Sreehari.etal2016}, where the proximal operator is simply replaced by a more general denoiser $\Dsf: \R^n \rightarrow \R^n$ such as BM3D~\cite{Dabov.etal2007} or DnCNN~\cite{Zhang.etal2017}. For example, the widely used PnP-ADMM can be summarized as
\begin{subequations}
\label{Eq:PnPADMM}
\begin{align}
&\zbm^k = \prox_{\gamma g}(\xbm^{k-1}+\sbm^{k-1})\\
&\xbm^k = \Dsf_\sigma(\zbm^k - \sbm^{k-1})\\
&\sbm^k = \sbm^{k-1} + \xbm^k - \zbm^k,
\end{align}
\end{subequations}
where, in analogy with $\tau > 0$ in~\eqref{Eq:ProximalOperator}, we introduce the parameter $\sigma > 0$ controlling the relative strength of the denoiser. Remarkably, this heuristic of using denoisers not associated with any $h$ within an iterative algorithm exhibited great empirical success~\cite{Zhang.etal2017a, Dong.etal2019, Zhang.etal2019, Ahmad.etal2020} and spurred a great deal of theoretical work on PnP algorithms~\cite{Chan.etal2016, Meinhardt.etal2017, Buzzard.etal2017, Sun.etal2018a, Tirer.Giryes2019, Teodoro.etal2019, Ryu.etal2019}.

An elegant fixed-point convergence analysis of PnP-ADMM was recently presented in~\cite{Ryu.etal2019}. By substituting $\vbm^k = \zbm^k - \sbm^{k-1}$ into PnP-ADMM, the algorithm is expressed in terms of an operator
\begin{equation}
\label{Eq:DROperator}
\Psf \defn \frac{1}{2}\Isf + \frac{1}{2}(2\Gsf-\Isf)(2\Dsf_\sigma-\Isf) \quad\text{with}\quad \Gsf \defn \prox_{\gamma g},
\end{equation}
where $\Isf$ denotes the identity operator. The convergence of PnP-ADMM is then established through its equivalence to the fixed-point convergence of the sequence $\vbm^k = \Psf(\vbm^{k-1})$. The equivalence of PnP-ADMM to the iterations of the operator~\eqref{Eq:DROperator} originates from the well-known relationship between ADMM and the Douglas-Rachford splitting~\cite{Parikh.Boyd2014, Eckstein.Bertsekas1992, Buzzard.etal2017, Ryu.etal2019}.

Scalable optimization algorithms have become increasingly important in the context of large-scale problems arising in machine learning and data science~\cite{Bottou.etal2018}. Stochastic and online optimization techniques have been investigated for traditional ADMM~\cite{Wang.Banerjee2012, Ouyang.etal2013, Suzuki2013, Zhong.etal2014, Huang.etal2019}, where $\prox_{\gamma g}$ is approximated using a subset of observations (with or without subsequent linearization). Our work contributes to this area by investigating the scalability of PnP-ADMM that is \emph{not} minimizing any explicit objective function. Since PnP-ADMM can integrate powerful deep neural net denoisers, there is a need to understand its theoretical properties and ability to address large-scale imaging problems.

Before introducing our algorithm, it is worth briefly mentioning an emerging paradigm of using deep neural nets for solving ill-posed imaging inverse problems (see, reviews~\cite{McCann.etal2017, Lucas.etal2018, Knoll.etal2020, Ongie.etal2020}). This work is most related to techniques that explicitly decouple the measurement model from the learned prior. For example, learned denoisers have been adopted for a class of algorithms in compressive sensing known as \emph{approximate message passing (AMP)}~\cite{Tan.etal2015, Metzler.etal2016a, Metzler.etal2016, Fletcher.etal2018}. The key difference of PnP from AMP is that it does not assume random measurement operators. \emph{Regularization by denoising (RED)} is a closely related method that specifies an explicit regularizers that has a simple gradient~\cite{Romano.etal2017, Bigdeli.etal2017, Sun.etal2019b, Mataev.etal2019}. PnP does not seek the existence of such an objective. Instead interpreting solutions as equilibrum points balancing the data-fit and the prior~\cite{Buzzard.etal2017}. Finally, a recent line of work has investigated the recovery and convergence guarantees for priors specified by \emph{generative adversarial networks (GANs)}~\cite{Bora.etal2017, Shah.Hegde2018, Hyder.etal2019, Raj.etal2019, Latorre.etal2019}. PnP does not seek to project its iterates to the range of a GAN, instead it directly uses the output of a simple AWGN denoiser to improve the estimation quality. This simplifies the training and application of learned priors within the PnP methodology. Our work contributes to this broad area by providing new conceptual, theoretical, and empirical insights into incremental ADMM optimization under statistical priors specified as deep neural net denoisers.

\section{Incremental PnP-ADMM}
\label{Sec:Algorithm}

\begin{algorithm}[t]
        \caption{Incremental Plug-and-Play ADMM (IPA)}\label{Alg:IPA}
        \begin{algorithmic}[1]
                \State \textbf{input: } initial values $\xbm^0, \sbm^0 \in \R^n$, parameters $\gamma, \sigma > 0$.
                \For{$k = 1, 2, 3, \dots$}
                \State Choose an index $i_k \in \{1,\dots,b\}$
                \State $\zbm^k \leftarrow \Gsf_{i_k}(\xbm^{k-1}+\sbm^{k-1})$ where $\Gsf_{i_k} \defn \prox_{\gamma g_{i_k}}$
                \State $\xbm^k \leftarrow \Dsf_\sigma(\zbm^k-\sbm^{k-1})$
                \State $\sbm^k \leftarrow \sbm^{k-1}+\xbm^k-\zbm^k$
                \EndFor
        \end{algorithmic}
\end{algorithm}%

Batch PnP algorithms operate on the whole observation vector $\ybm \in \R^m$. We are interested in partial randomized processing of observations by considering the decomposition of $\R^m$ into $b \geq 1$ blocks
\begin{equation*}
\R^m = \R^{m_1} \times \R^{m_2} \times \cdots \times \R^{m_b} \quad\text{with}\quad m = m_1 + m_2 +  \cdots + m_b.
\end{equation*}
We thus consider data-fidelity terms of the form
\begin{equation}
\label{Eq:DataFidelityComponents}
g(\xbm) = \frac{1}{b} \sum_{i = 1}^b g_i(\xbm), \quad \xbm \in \R^n,
\end{equation}
where each $g_i$ is evaluated only on the subset $\ybm_i \in \R^{m_i}$ of the full data $\ybm$.

PnP-ADMM is often impractical when $b$ is very large due to the complexity of computing $\prox_{\gamma g}$. As shown in Algorithm~\ref{Alg:IPA}, the proposed IPA algorithm extends stochastic variants of traditional ADMM~\cite{Wang.Banerjee2012, Ouyang.etal2013, Suzuki2013, Zhong.etal2014, Huang.etal2019} by integrating denoisers $\Dsf_\sigma$ that are \emph{not} associated with any $h$. Its per-iteration complexity is independent of the number of data blocks $b$, since it processes only a single component function $g_i$ at every iteration.

In principle, IPA can be implemented using different  block selection rules. The strategy adopted for our theoretical analysis focuses on the usual strategy of selecting indices $i_k$ as independent and identically distributed (i.i.d.) random variables distributed uniformly over $\{1, \dots, b\}$. An alternative would be to proceed in epochs of $b$ consecutive iterations, where at the start of each epoch the set $\{1, \dots, b\}$ is reshuffled, and $i_k$ is selected from this ordered set~\cite{Bertsekas2011}. In some applications, it might also be beneficial to select indices $i_k$ in an online data-adaptive fashion by taking into account the statistical relationships among observations~\cite{Tian.etal2015, Kellman.etal2020}.

Unlike PnP-SGD, IPA does not require smoothness of the functions $g_i$. Instead of computing the partial gradient $\nabla g_i$, as is done in PnP-SGD, IPA evaluates the partial proximal operator $\Gsf_i$. Thus, the maximal benefit of IPA is expected for problems in which $\Gsf_i$ is efficient to evaluate. This is a case for a number of functions commonly used in computational imaging, compressive sensing, and machine learning (see the extensive discussion on proximal operators in~\cite{Beck2017a}).

Let us discuss two widely used scenarios. The proximal operator of the $\ell_2$-norm data-fidelity term ${g_i(\xbm)= \frac{1}{2}\|\ybm_i-\Abm_i\xbm\|_2^2}$ has a closed-form solution
\begin{equation}
\label{Eq:ProximalExample}
\Gsf_i(\zbm) = \prox_{\gamma g_i}(\zbm) = \left(\Ibf+\gamma\Abm_i^\Tsf\Abm_i\right)^{-1}\left(\zbm+\gamma\Abm_i^\Tsf\ybm\right)
\end{equation}
for $\gamma > 0$ and $\zbm \in \R^n$. Prior work has extensively discussed efficient strategies for evaluating~\eqref{Eq:ProximalExample} for a variety of linear operators, including convolutions, partial Fourier transforms, and subsampling masks~\cite{Afonso.etal2010, Wohlberg2016, Ramani.Fessler2012a, Almeida.Figueiredo2013}. As a second example, consider the $\ell_1$-data fidelity term $g_i(\xbm) = \|\ybm_i-\Abm_i\xbm\|_1$, which is nonsmooth. The corresponding proximal operator has a closed form solution for any orthogonal operator $\Abm_i$ and can also be efficiently computed in many other settings~\cite{Beck2017a}. We numerically evaluate the effectiveness of IPA on both $\ell_1$- and $\ell_2$-norm data-fidelity terms and deep neural net priors in Section~\ref{Sec:Experiments}.

IPA can also be implemented as a \emph{minibatch} algorithm, processing several blocks in parallel at every iteration, thus improving its efficiency on multi-processor hardware architectures. Algorithm~\ref{Alg:MinibatchIPA} presents the minibatch version of IPA that averages several proximal operators evaluated over different data blocks. When the minibatch size $p = 1$, Algorithm~\ref{Alg:MinibatchIPA} reverts to Algorithm~\ref{Alg:IPA}. The main benefit of minibatch IPA is its suitability for parallel computation of $\Gsfhat$, which can take advantage of multi-processor architectures.

Minibatch IPA is related to the \emph{proximal average} approximation of $\Gsf = \prox_{\gamma g}$~\cite{Bauschke.etal2008, Yu2013}
\begin{equation*}
\overline{\Gsf}(\xbm) = \frac{1}{b} \sum_{i = 1}^b \prox_{\gamma g_i }(\xbm), \quad \xbm \in \R^n.
\end{equation*}
When Assumption~\ref{As:WeaklyConvexData} is satisfied, then the approximation error is bounded for any $\xbm \in \R^n$ as~\cite{Yu2013}
\begin{equation*}
\|\Gsf(\xbm)-\overline{\Gsf}(\xbm)\| \leq 2 \gamma L.
\end{equation*}
Minibatch IPA thus simply uses a minibatch approximation $\Gsfhat$ of the proximal average $\overline{\Gsf}$. One implication of this is that even when the minibatch is \emph{exactly} equal to the full measurement vector, minibatch IPA is not exact due to the approximation error introduced by the proximal average. However, the resulting approximation error can be made as small as desired by controlling the penalty parameter $\gamma > 0$.

\begin{algorithm}[t]
        \caption{Minibatch IPA}\label{Alg:MinibatchIPA}
        \begin{algorithmic}[1]
                \State \textbf{input: } initial values $\xbm^0, \sbm^0 \in \R^n$, parameters $\gamma, \sigma > 0$, minibatch size $p \geq 1$.
                \For{$k = 1, 2, 3, \dots$}
                \State Choose indices $i_1, \dots, i_p$ from the set $\{1,\dots,b\}$.
                \State $\zbm^k \leftarrow \Gsfhat(\xbm^{k-1}+\sbm^{k-1})$ where
                $\Gsfhat\defn \frac{1}{p}\sum_{j = 1}^p \prox_{\gamma g_{i_j}}$
                \State $\xbm^k \leftarrow \Dsf_\sigma(\zbm^k-\sbm^{k-1})$
                \State $\sbm^k \leftarrow \sbm^{k-1}+\xbm^k-\zbm^k$
                \EndFor
        \end{algorithmic}
\end{algorithm}%

\section{Theoretical Analysis}
\label{Sec:Theory}

We now present a theoretical analysis of IPA. We fist present an intuitive interpretation of its solutions, and then present our convergence analysis under a set of explicit assumptions.

\subsection{Fixed Point Interpretation}
\label{Sec:FixPointInterp}

IPA cannot be interpreted using the standard tools from convex optimization, since its solution is generally not a minimizer of an objective function. Nonetheless, we develop an intuitive operator based interpretation (see Appendix~\ref{Sec:Equilibrium} for additional details).

Consider the following set-valued operator
\begin{equation}
\label{Eq:SolutionSet}
\Tsf \defn \gamma \partial g + (\Dsf_\sigma^{-1} - \Isf), \quad \gamma > 0,
\end{equation}
where $\partial g$ is the subdifferential of the data-fidelity term and $\Dsf_\sigma^{-1}(\xbm) \defn \{\zbm \in \R^n : \xbm = \Dsf_\sigma(\zbm)\}$ is the inverse operator of the denoiser $\Dsf_\sigma$. Note that this inverse operator exists even when $\Dsf_\sigma$ is not one-to-one~\cite{Eckstein.Bertsekas1992, Ryu.Boyd2016}. By characterizing the fixed points of PnP algorithms, it can be shown that their solutions can be interpreted as vectors in the zero set of $\Tsf$
\begin{align*}
&\zerobm \in \Tsf(\xbmast) = \gamma \partial g(\xbmast) + (\Dsf_\sigma^{-1}(\xbmast) - \xbmast) \\
&\Leftrightarrow \quad
 \xbmast \in \zer(\Tsf) \defn \{\xbm \in \R^n : \zerobm \in \Tsf(\xbm) \}.
\end{align*}
Consider the following two sets
\begin{align*}
&\zer(\partial g) \defn \{\xbm \in \R^n : \zerobm \in \partial g(\xbm)\}\quad\text{and}\\
&\fix(\Dsf_\sigma) \defn \{\xbm \in \R^n : \xbm = \Dsf_\sigma(\xbm)\},
\end{align*}
where $\zer(\partial g)$ is the set of all critical points of the data-fidelity term and $\fix(\Dsf_\sigma)$ is the set of all fixed points of the denoiser. Intuitively, the fixed points of $\Dsf_\sigma$ correspond to all vectors that are \emph{not} denoised, and therefore can be interpreted as vectors that are \emph{noise-free} according to the denoiser.

If $\xbmast \in \zer(\partial g) \cap \fix(\Dsf_\sigma)$, then $\xbmast \in \zer(\Tsf)$, which implies that $\xbmast$ is one of the solutions. Hence, any vector that minimizes a convex data-fidelity term $g$ and noiseless according to $\Dsf_\sigma$ is in the solution set. On the other hand, when $\zer(\partial g)\cap \fix(\Dsf_\sigma) = \varnothing$, then $\xbmast \in \zer(\Tsf)$ corresponds to an equilibrium point between two sets.

This interpretation of PnP highlights one important aspect that is often overlooked in the literature, namely that, unlike in the traditional formulation~\eqref{Eq:RegularizedOptimization}, the regularization in PnP depends on both the denoiser parameter $\sigma > 0$ and the penalty parameter $\gamma > 0$, with both influencing the solution. Hence, the best performance is obtained by jointly tuning both parameters for a given experimental setting. In the special case of $\Dsf_\sigma = \prox_{\gamma h}$ with $\gamma = \sigma^2$, we have
\begin{align*}
&\fix(\Dsf_\sigma) = \{\xbm \in \R^n: \zerobm \in \partial h(\xbm)\} \quad\text{and}\\
&\zer(\Tsf) \defn \{\xbm \in \R^n : \zerobm \in \partial g(\xbm) + \partial h(\xbm) \},
\end{align*}
which corresponds to the optimization formulation~\eqref{Eq:RegularizedOptimization} whose solutions are independent of $\gamma$.

\subsection{Convergence Analysis}

Our analysis requires three assumptions that jointly serve as sufficient conditions.

\medskip\noindent
\begin{assumption}
\label{As:WeaklyConvexData}
Each $g_i$ is proper, closed, convex, and Lipschitz continuous with constant $L_i > 0$. We define the largest Lipschitz constant as $L = \max\{L_1, \dots, L_b\}$.
\end{assumption}
This assumption is commonly adopted in nonsmooth optimization and is equivalent to existence of a global upper bound on subgradients~\cite{Boyd.Vandenberghe2008, Yu2013, Ouyang.etal2013}. It is satisfied by a large number of functions, such as the $\ell_1$-norm. The $\ell_2$-norm also satisfies Assumption~\ref{As:WeaklyConvexData} when it is evaluated over a bounded subset of $\R^n$. We next state our assumption on $\Dsf_\sigma$.

\medskip\noindent
\begin{assumption}
\label{As:AveragedDenoiser}
The residual $\Rsf_\sigma \defn \Isf - \Dsf_\sigma$ of the denoiser $\Dsf_\sigma$ is firmly nonexpansive.
\end{assumption}
We review firm nonexpansiveness and other related concepts in the Appendix~\ref{Sec:Equilibrium}.
Firmly nonexpansive operators are a subset of \emph{nonexpansive} operators (those that are Lipschitz continuous with constant one). A simple strategy to obtain a firmly nonexpansive operator is to create a $(1/2)$-averaged operator from a nonexpansive operator~\cite{Parikh.Boyd2014}. The residual $\Rsf_\sigma$ is firmly nonexpansive \emph{if and only if} $\Dsf_\sigma$ is firmly nonexpansive, which implies that the proximal operator automatically satisfies Assumption~\ref{As:AveragedDenoiser}~\cite{Parikh.Boyd2014}.

The rationale for stating Assumption~\ref{As:AveragedDenoiser} for $\Rsf_\sigma$ is based on our interest in \emph{residual} deep neural nets. The success of residual learning in the context of image restoration is well known~\cite{Zhang.etal2017}. Prior work has also shown that Lipschitz constrained residual networks yield excellent performance without sacrificing stable convergence~\cite{Ryu.etal2019, Sun.etal2019b}. Additionally, there has recently been an explosion of techniques for training Lipschitz constrained and firmly nonexpansive deep neural nets~\cite{Ryu.etal2019, Terris.etal2020, Miyato.etal2018, Fazlyab.etal2019}.

\medskip\noindent
\begin{assumption}
\label{As:ExistenceOfFixed}
The operator $\Tsf$ in~\eqref{Eq:SolutionSet} is such that $\zer(\Tsf) \neq \varnothing$. There also exists $R < \infty$ such that
\begin{equation*}
\|\xbm^k-\xbmast\|_2
\leq R \quad\text{for all}\quad \xbmast \in \zer(\Tsf).
\end{equation*}
\end{assumption}

The first part of the assumptions simply ensures the existence of a solution. The existence of the bound $R$ often holds in practice, as many denoisers have bounded range spaces. In particular, this is true for a number of image denoisers whose outputs live within the bounded subset $[0, 255]^n \subset \R^n$.

We will state our convergence results in terms of the operator $\Ssf: \R^n \rightarrow \R^n$ defined as
\begin{equation}
\label{Eq:DRSMonOperator}
\Ssf \defn \Dsf_\sigma-\Gsf(2\Dsf_\sigma-\Isf).
\end{equation}
Both IPA and traditional PnP-ADMM can be interpreted as algorithms for computing an element in $\zer(\Ssf)$, which is equivalent to finding an element of $\zer(\Tsf)$ (see details in Appendix~\ref{Sec:Equilibrium}).

We are now ready to state our main result on IPA.
\begin{theorem}
\label{Thm:ConvThm1}
Run IPA for $t \geq 1$ iterations with random i.i.d.~block selection under Assumptions~\ref{As:WeaklyConvexData}-\ref{As:ExistenceOfFixed} using a fixed penalty parameter $\gamma > 0$. Then, the sequence $\vbm^k = \zbm^k-\sbm^{k-1}$ satisfies
\begin{equation}
\label{Eq:MainBound}
\E\left[\frac{1}{t}\sum_{k = 1}^t \|\Ssf(\vbm^k)\|_2^2\right] \leq \frac{(R+2\gamma L)^2}{t} + \max\{\gamma, \gamma^2\} C,
\end{equation}
where $C \defn 4 L R + 12 L^2$ is a positive constant.
\end{theorem}

In order to contextualize this result, we also review the convergence of the traditional PnP-ADMM.

\medskip
\begin{theorem}
\label{Thm:ConvThm2}
Run PnP-ADMM for $t \geq 1$ iterations under Assumptions~\ref{As:WeaklyConvexData}-\ref{As:ExistenceOfFixed} using a fixed penalty parameter $\gamma > 0$. Then, the sequence $\vbm^k = \zbm^k-\sbm^{k-1}$ satisfies
\begin{equation}
\frac{1}{t}\sum_{k = 1}^t \|\Ssf(\vbm^k)\|_2^2 \leq \frac{(R+2\gamma L)^2}{t}.
\end{equation}
\end{theorem}

Both proofs are provided in the Appendix~\ref{Sec:ProofThm1}. The proof of Theorem~\ref{Thm:ConvThm2} is a modification of the analysis in~\cite{Ryu.etal2019}, obtained by relaxing the \emph{strong convexity} assumption in~\cite{Ryu.etal2019} by Assumption~\ref{As:WeaklyConvexData} and replacing the assumption that $\Rsf_\sigma$ is a \emph{contraction} in~\cite{Ryu.etal2019} by Assumption~\ref{As:AveragedDenoiser}. Theorem~\ref{Thm:ConvThm2} establishes that the iterates of PnP-ADMM satisfy $\|\Ssf(\vbm^t)\| \rightarrow 0$ as $t \rightarrow \infty$. Since $\Ssf$ is firmly nonexpansive and $\Dsf_\sigma$ is nonexpansive, the Krasnosel'skii-Mann theorem (see Section 5.2 in~\cite{Bauschke.Combettes2017}) directly implies that $\vbm^t \rightarrow \zer(\Ssf)$ and $\xbm^t = \Dsf_\sigma(\vbm^t) \rightarrow \zer(\Tsf)$.

Theorem~\ref{Thm:ConvThm1} establishes that IPA approximates the solution obtained by the full PnP-ADMM up to an error term that depends on the penalty parameter $\gamma$. One can precisely control the accuracy of IPA by setting $\gamma$ to a desired level. In practice, $\gamma$ can be treated as a hyperparameter and tuned to maximize performance for a suitable image quality metric, such as SNR or SSIM. Our numerical results in Section~\ref{Sec:Experiments} corroborate that excellent SNR performance of IPA can be achieved without taking $\|\Ssf(\vbm^t)\|_2$ to zero, which simplifies practical applicability of IPA.  (Note that the convergence analysis for IPA in Theorem~\ref{Thm:ConvThm1} can be easily extended to minibatch IPA with a straightforward extension of Lemma~\ref{Sup:Lem:NoiseBound} in Appendix~\ref{Sup:Sec:LemmasThm1} to several indices, and by following the steps of the main proof in Appendix~\ref{Sup:Sec:MainProof}.)

\begin{figure*}[t]
        \centering\includegraphics[width=16.5cm]{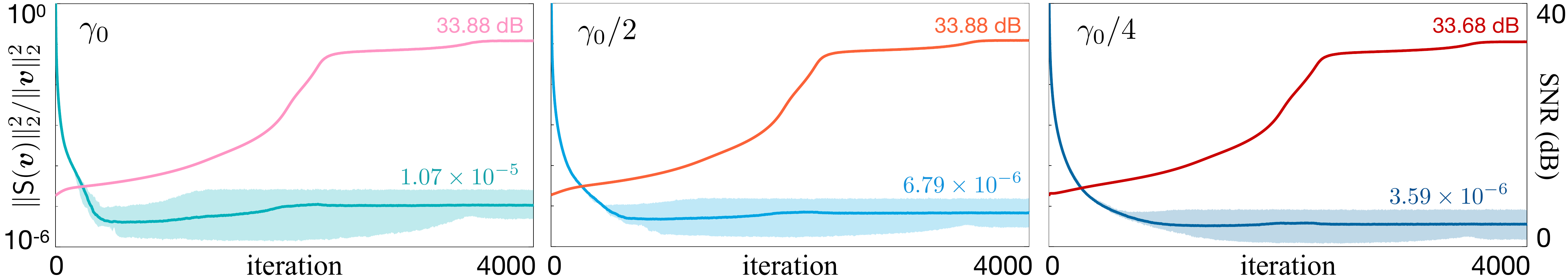}
        \caption{Illustration of the influence of the penalty parameter $\gamma > 0$ on the convergence of IPA for a DnCNN prior. The average normalized distance to $\zer(\Ssf)$ and SNR (dB) are plotted against the iteration number with the shaded areas representing the range of values attained over $12$ test images. The accuracy of IPA improves for smaller values of $\gamma$. However, the SNR performance is nearly identical, indicating that in practice IPA can achieve excellent results for a range of fixed $\gamma$ values.}
        \label{fig:l1_convergence}
\end{figure*}

Finally, note that our analysis can be also performed under assumptions adopted in~\cite{Ryu.etal2019}, namely that $g_i$ are strongly convex and $\Rsf_\sigma$ is a contraction. Such an analysis leads to the statement
\begin{equation}
\label{Eq:MainBoundStrong}
\E\left[\|\xbm^t-\xbmast\|_2\right] \leq \eta^t  (2R+4\gamma L) +  (4\gamma L)/(1-\eta),
\end{equation}
where $0<\eta<1$. Equation~\eqref{Eq:MainBoundStrong} establishes a linear convergence to $\zer(\Tsf)$ up to an error term. A proof of~\eqref{Eq:MainBoundStrong} is provided in the Appendix~\ref{Sec:StrongConvex}. As corroborated by our simulations in Section~\ref{Sec:Experiments}, the actual convergence of IPA holds even more broadly than suggested by both sets of sufficient conditions. This motivates further analysis of IPA under more relaxed assumptions that we leave to future work.

\section{Numerical Validation}
\label{Sec:Experiments}
Recent work has shown the excellent performance of PnP algorithms for smooth data-fidelity terms using advanced denoising priors. Our goal in this section is to extend these studies with simulations validating the effectiveness of IPA for nonsmooth data-fidelity terms and deep neural net priors, as well as demonstrating its scalability to large-scale inverse problems. We consider two applications of the form $\ybm=\Abm\xbm+\ebm$, where $\ebm\in\R^m$ denotes the noise and $\Abm\in\R^{m\times n}$ denotes either a random Gaussian matrix in \emph{compressive sensing (CS)} or the transfer function in \emph{intensity diffraction tomography}~\cite{Ling.etal18}.

Our deep neural net prior is based on the DnCNN architecture~\cite{Zhang.etal2017}, with its batch normalization layers removed for controlling the Lipschitz constant of the network via spectral normalization~\cite{Sedghi.etal2019}. 
We train a nonexpansive residual network $\Rsf_\sigma$ by predicting the noise residual from its noisy input. This means that $\Rsf_\sigma$ satisfies the necessary condition for firm nonexpansiveness of $\Dsf_\sigma$. The training data is generated by adding AWGN to the images from the BSD400 dataset~\cite{Martin.etal2001}. The reconstruction quality is quantified using the signal-to-noise ratio (SNR) in dB. We pre-train several deep neural net models as denoisers for $\sigma \in [1, 10]$, using $\sigma$ intervals of $0.5$, and use the denoiser achieving the best SNR.

\subsection{Integration of Nonsmooth Data-Fidelity Terms and Pretrained Deep Priors}

We first validate the effectiveness of Theorem~\ref{Thm:ConvThm1} for non-smooth data-fidelity terms. The matrix $\Abm$ is generated with i.i.d. zero-mean Gaussian random elements of variance $1/m$, and $\ebm$ as a sparse Bernoulli-Gaussian vector with the sparsity ratio of 0.1. This means that, in expectation, ten percent of the elements of $\ybm$ are contaminated by AWGN. The sparse nature of noise motivates the usage of the $\ell_1$-norm $g(\xbm)=\|\ybm-\Abm\xbm\|_1$, since it can effectively mitigate outliers. The nonsmoothness of $\ell_1$-norm prevents the usage of gradient-based algorithms such as PnP-SGD. On the other hand, the application IPA is facilitated by efficient strategies for computing the proximal operator~\cite{Chambolle.2004,Beck.Teboulle2009a}.

We set the measurement ratio to be approximately $m/n=0.7$ with AWGN of standard deviation $5$. Twelve standard images from \emph{Set 12} are used in testing, each resized to $64\times 64$ pixels for rapid parameter tuning and testing. We quantify the convergence accuracy using the normalized distance $\|\Ssf(\vbm^k)\|_2^2/\|\vbm^k\|_2^2$, which is expected to approach zero as IPA converges to a fixed point.

Theorem~\ref{Thm:ConvThm1} characterizes the convergence of IPA in terms of $\|\Ssf(\vbm^k)\|_2$ up to a constant error term that depends on $\gamma$. This is illustrated in Fig.~\ref{fig:l1_convergence} for three values of the penalty parameter $\gamma\in\{\gamma_0,\gamma_0/2,\gamma_0/4\}$ with $\gamma_0=0.02$. The average normalized distance $\|\Ssf(\vbm^k)\|_2^2/\|\vbm^k\|_2^2$ and SNR are plotted against the iteration number and labeled with their respective final values. The shaded areas represent the range of values attained across all test images. IPA is implemented to use a random half of the elements in $\ybm$ in every iteration to impose the data-consistency. Fig.~\ref{fig:l1_convergence} shows the improved convergence of IPA to $\zer(\Ssf)$ for smaller values of $\gamma$, which is consistent with our theoretical analysis. Specifically, the final accuracy improves approximately $3\times$ (from $1.07\times10^{-5}$ to $3.59\times10^{-6}$) when $\gamma$ is reduced from $\gamma_0$ to $\gamma_0/4$. On the other hand, the SNR values are nearly identical for all three experiments, indicating that in practice different $\gamma$ values lead to fixed points of similar quality. This indicates that IPA can achieve high-quality result without taking $\|\Ssf(\vbm^k)\|_2$ to zero.

\begin{figure*}[t]
        \centering\includegraphics[width=\linewidth]{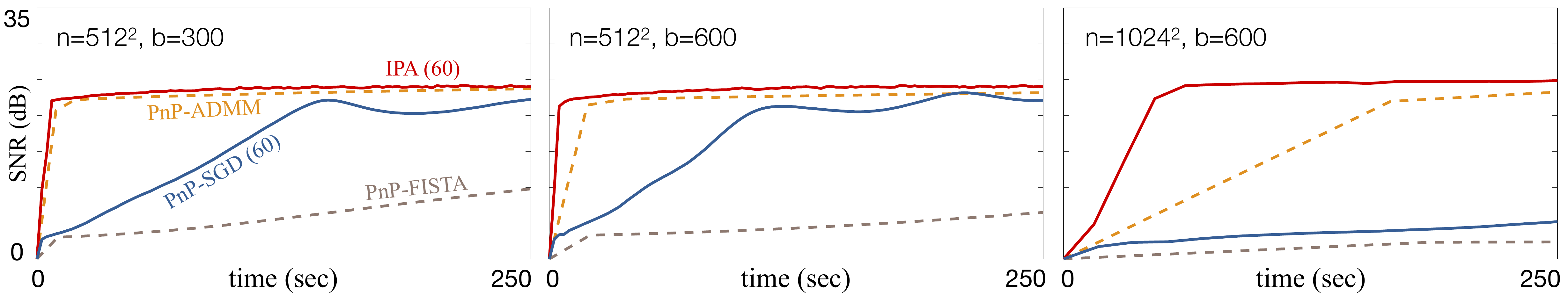}
        \caption{Illustration of scalability of IPA and several widely used PnP algorithms on problems of different sizes. The parameters $n$ and $b$ denote the image size and the number of acquired intensity images, respectively. The average SNR is plotted against time in seconds. Both IPA and PnP-SGD use random minibatches of 60 measurements at every iteration, while PnP-ADMM and PnP-FISTA use all the measurements. The figure highlights the fast empirical convergence of IPA compared to PnP-SGD as well as its ability to address larger problems compared to PnP-ADMM and PnP-FISTA.
        }
        \label{Fig:Scalabiity}
\end{figure*}
\begin{table*}[t]
        \centering
        \caption{Final average SNR (dB) and Runtime obtained by several PnP algorithms on all test images.}
        \label{Tab:Speed}
        \small
        \vspace{5pt}
        \begin{tabular*}{505pt}{C{90pt}C{35pt}C{35pt}C{90pt}C{90pt}C{90pt}C{0pt}C{0pt}}
                \toprule
                \multirow{2}{*}{\textbf{Simulations}} & \multicolumn{2}{c}{\multirow{2}{*}{\textbf{Parameters}}} & \multicolumn{1}{c}{$n=512^2$} & \multicolumn{1}{c}{$n=512^2$} & \multicolumn{1}{c}{$n=1024^2$} & \\
                & & & \multicolumn{1}{c}{($b$ = 300)} & \multicolumn{1}{c}{($b$ = 600)} & \multicolumn{1}{c}{($b$ = 600)} & \\
                \cmidrule(lr){2-3} \cmidrule(lr){4-4} \cmidrule(lr){5-5} \cmidrule(lr){6-6}
                \textbf{Algorithms} & $\sigma$ & $\gamma$ & \multicolumn{3}{c}{SNR in dB (Runtime)} \\
                \midrule \\ [-2.1ex]
                PnP-FISTA & 1 & 5$\times 10^{-4}$ & 22.60 (19.4 min) & 22.79 (42.6 min) & 23.56 (8.1 hr) &   \\ [0.7ex]
                \cdashline{1-6} \\ [-1.4ex]
                PnP-SGD ($60$)  & 1 & 5$\times 10^{-4}$ & 22.31 (7.1 min) & 22.74 (5.2 min) & 23.42 (44.3 min) &  \\ [0.3ex]
                \cdashline{1-6} \\ [-1.4ex]
                PnP-ADMM & 2.5 & 1 & \textbf{24.23} (7.4 min) & \textbf{24.40} (14.7 min) & \textbf{25.50} (1.4 hr) &  \\ [0.7ex]
                \cdashline{1-6} \\ [-1.4ex]
                \proposed~($60$) & 2.5 & 1 & \colorbox{lightgreen}{23.65 (\textbf{1.7 min})} & \colorbox{lightgreen}{23.88 (\textbf{2 min})} & \colorbox{lightgreen}{24.95 (\textbf{11 min})} &   \\
                \bottomrule
        \end{tabular*}
\end{table*}

\subsection{Scalability in Large-scale Optical Tomography}
We now discuss the scalability of IPA on intensity diffraction tomography, which is a data intensive computational imaging modality~\cite{Ling.etal18}. The goal is to recover the spatial distribution of the \emph{complex} permittivity contrast of an object given a set of its intensity-only measurements. In this problem, $\Abm$ consists of a set of $b$ complex matrices $[\Abm_1,\dots,\Abm_b]^\Tsf$, where each $\Abm_i$ is a convolution corresponding to the $i$th measurement $\ybm_i$. We adopt the $\ell_2$-norm loss ${g(\xbm)=\|\ybm-\Abm\xbm\|_2^2}$ as the data-fidelity term to empirically compare the performance of IPA and PnP-SGD on the same problem.

\begin{table*}[t]
        \centering
        \caption{Per-iteration memory usage specification for reconstructing 1024$\times$1024 images}
        \label{Tab:Memory}
        \vspace{5pt}
        \small
        \begin{tabular*}{500pt}{C{40pt}C{40pt} C{100pt}C{72pt} C{100pt}C{72pt}C{0pt}}
                \toprule
                \multicolumn{2}{c}{\textbf{Algorithms}} & \multicolumn{2}{c}{PnP-ADMM} & \multicolumn{2}{c}{\proposed~(\textbf{Ours})} \\
                \cmidrule(lr){3-4} \cmidrule(lr){5-6}
                \multicolumn{2}{c}{\textbf{Variables}} &  size & memory & size & memory   \\
                \midrule \\ [-2ex]
                \multirow{2}{*}{$\{\Abm_i\}$}   & real & $1024\times1024\times600$ & 9.38 GB & $1024\times1024\times60$ & 0.94 GB \\ [0.5ex]
                \cdashline{3-6} \\ [-2ex]
                & imaginary & $1024\times1024\times600$ & 9.38 GB & $1024\times1024\times60$ & 0.94 GB \\ [0.5ex]
                \cdashline{3-6} \\ [-2ex]
                \multicolumn{2}{c}{$\{\ybm_i\}$} & $1024\times1024\times600$  & 18.75 GB & $1024\times1024\times60$  & 1.88 GB \\ [0.5ex]
                \cdashline{3-6} \\ [-2ex]
                \multicolumn{2}{c}{others combined} & ---  & 0.13 GB & --- & 0.13 GB \\
                \midrule
                \multicolumn{2}{c}{\textbf{Total}} & & \textbf{37.63 GB} &  & \colorbox{lightgreen}{\textbf{3.88 GB}}  \\
                \bottomrule
        \end{tabular*}
\end{table*}

In the simulation, we follow the experimental setup in~\cite{Ling.etal18} under AWGN corresponding to an input SNR of 20 dB. We select six images from the CAT2000 dataset~\cite{Borji.etal2015} as our test examples, each cropped to $n$ pixels. We assume real permittivity functions, but still consider complex valued measurement operator $\Abm$ that accounts for both absorption and phase~\cite{Ling.etal18}. Due to the large size of data, we process the measurements in epochs using minibatches of size 60.

Fig.~\ref{Fig:Scalabiity} illustrates the evolution of average SNR against runtime for several PnP algorithms, namely PnP-ADMM, PnP-FISTA, PnP-SGD, and IPA, for images of size $n \in \{512\times512, 1024\times1024\}$ and the total number of intensity measurements
$b\in\{300, 600\}$. The final values of SNR as well as the total runtimes are summarized in Table~\ref{Tab:Speed}. The table highlights the overall best SNR performance in bold and the shortest runtime in light-green. In every iteration, PnP-ADMM and PnP-FISTA use all the measurements, while IPA and PnP-SGD use only a small subset of 60 measurements. IPA thus retains its effectiveness for large values of $b$, while batch algorithms become significantly slower. Moreover, the scalability of IPA over PnP-ADMM becomes more notable when the image size increases. For example, Table~\ref{Tab:Speed} highlights the convergence of IPA to 24.95 dB within 11 minutes, while PnP-ADMM takes 1.4 hours to reach a similar SNR value. Note the rapid progress of PnP-ADMM in the first few iterations, followed by a slow but steady progress until its convergence to the values reported in Table~\ref{Tab:Speed}. This behavior of ADMM is well known and has been widely reported in the literature (see \emph{Section 3.2.2 ``Convergence in Practice''} in~\cite{Boyd.etal2011}). We also observe faster convergence of IPA compared to both PnP-SGD and PnP-FISTA, further highlighting the potential of IPA to address large-scale problems where partial proximal operators are easy to evaluate.

Another key feature of IPA is its memory efficiency due to incremental processing of data. The memory considerations in optical tomography include the size of all the variables related to the desired image $\xbm$, the measured data $\{\ybm_i\}$, and the variables related to the forward model $\{\Abm_i\}$. Table~\ref{Tab:Memory} records the total memory (GB) used by IPA and PnP-ADMM for reconstructing a $1024\times1024$ pixel permittivity image, with the smallest value highlighted in light-green. PnP-ADMM requires 37.63 GB of memory due to its batch processing of the whole dataset, while IPA uses only 3.88 GB---nearly \emph{one-tenth} of the former---by adopting incremental processing of data. In short, our numerical evaluations highlight both fast and stable convergence and flexible memory usage of IPA in the context of large-scale optical tomographic imaging.

\section{Conclusion}
\label{Sec:Conclusion}

This work provides several new insights into the widely used PnP methodology in the context of large-scale imaging problems. First, we have proposed IPA as a new incremental PnP algorithm. IPA extends PnP-ADMM to randomized partial processing of measurements and extends traditional optimization-based ADMM by integrating pre-trained deep neural nets. Second, we have theoretically analyzed IPA under a set of realistic assumptions, showing that IPA can approximate PnP-ADMM to a desired precision by controlling the penalty parameter. Third, our simulations highlight the effectiveness of IPA for nonsmooth data-fidelity terms and deep neural net priors, as well as its scalability to large-scale imaging. We observed faster convergence of IPA compared to several baseline PnP methods, including PnP-ADMM and PnP-SGD, when partial proximal operators can be efficiently evaluated. IPA can thus be an effective alternative to existing algorithms for addressing large-scale imaging problems. For future work, we would like to explore strategies to further relax our assumptions and explore distributed variants of IPA to enhance its performance in parallel settings.

\appendix
We adopt monotone operator theory~\cite{Ryu.Boyd2016, Bauschke.Combettes2017} for a unified analysis of IPA. In Appendix~\ref{Sec:ProofThm1}, we present the convergence analysis of IPA. In Appendix~\ref{Sec:StrongConvex}, we analyze the convergence of the algorithm for strongly convex data-fidelity terms and contractive denoisers. In Appendix~\ref{Sec:Equilibrium}, we discuss interpretation of IPA's fixed-points from the perspective of monotone operator theory. For completeness, in Appendix~\ref{Sec:PnPADMMAnalysis}, we discuss the convergence results for traditional PnP-ADMM~\cite{Ryu.etal2019}.  Additionally, in Supplement~\ref{Sec:BackgroundMaterial}, we provide the background material used in our analysis. In Supplement~\ref{Sec:TechnicalDetails}, we provide additional technical details, omitted from the main paper due to space, such as the details on our deep neural net architecture and results of additional simulations.

For the sake of simplicity, we uses $\|\cdot\|$ to denote the standard $\ell_2$-norm in $\R^n$. We will also use $\Dsf(\cdot)$ instead of $\Dsf_\sigma(\cdot)$ to denote the denoiser, thus dropping the explicit notation for $\sigma$.

\subsection{Convergence Analysis of IPA}
\label{Sec:ProofThm1}

In this section, we present one of the main results in this paper, namely the convergence analysis of IPA. A fixed-point convergence of averaged operators is well-known under the name of Krasnosel'skii-Mann theorem (see Section 5.2 in~\cite{Bauschke.Combettes2017}) and was recently applied to the analysis of PnP-SGD~\cite{Sun.etal2018a}. Additionally, PnP-ADMM was analyzed for strongly convex data-fidelity terms $g$ and contractive residual denoisers $\Rsf_\sigma$~\cite{Ryu.etal2019}. Our analysis here extends these results to IPA by providing an explicit upper bound on the convergence of IPA. In Appendix~\ref{Sup:Sec:MainProof}, we present the main steps of the proof, while in Appendix~\ref{Sup:Sec:LemmasThm1} we prove two technical lemmas useful for our analysis.

\subsubsection{Proof of Theorem~\ref{Thm:ConvThm1}}
\label{Sup:Sec:MainProof}

Appendix~\ref{Sup:Eq:FirmNonExpS} establishes that $\Ssf$ defined in~\eqref{Eq:DRSMonOperator} is firmly nonexpansive. Consider any $\vbmast \in \zer(\Ssf)$ and any $\vbm \in \R^n$, then we have
\begin{align}
\label{Sup:Eq:MonRedIter}
&\|\vbm-\vbmast-\Ssf\vbm\|^2 \\
&\nonumber
= \|\vbm-\vbmast\| - 2(\Ssf\vbm-\Ssf\vbmast)^\Tsf(\vbm-\vbmast) + \|\Ssf\vbm\|^2 \\
&\nonumber
 \leq\|\vbm-\vbmast\|^2 - \|\Ssf\vbm\|^2,
\end{align}
where we used the firm nonexpansiveness of $\Ssf$ and $\Ssf\xbmast = \zerobm$. The direct consequence of~\eqref{Sup:Eq:MonRedIter} is that
\begin{equation*}
\|\vbm-\vbmast-\Ssf\vbm\| \leq \|\vbm-\vbmast\|.
\end{equation*}

We now consider the following two equivalent representations of IPA for some iteration $k \geq 1$
\begin{subequations}
\begin{align}
\label{Sup:Eq:StochDRSvsIPA_a}
&\begin{cases}
\zbm^k = \Gsf_{i_k}(\xbm^{k-1} + \sbm^{k-1})\\
\xbm^k = \Dsf(\zbm^k - \sbm^{k-1}) \\
\sbm^k = \sbm^{k-1} + \xbm^k - \zbm^k,
\end{cases}\\
\quad\Leftrightarrow\quad
\label{Sup:Eq:StochDRSvsIPA_b}
& \begin{cases}
\xbm^{k-1} = \Dsf(\vbm^{k-1})\\
\zbm^k = \Gsf_{i_k}(2\xbm^{k-1}-\vbm^{k-1})\\
\vbm^k = \vbm^{k-1} + \zbm^k - \xbm^{k-1}
\end{cases}
\end{align}
\end{subequations}
where $i_k$ is a random variable uniformly distributed over $\{1, \dots, b\}$, $\Gsf_i = \prox_{\gamma g_i}$ is the proximal operator with respect to $g_i$, and $\Dsf$ is the denoiser. To see the equivalence between~\eqref{Sup:Eq:StochDRSvsIPA_a} and~\eqref{Sup:Eq:StochDRSvsIPA_b}, simply introduce the variable $\vbm^k = \zbm^k - \sbm^{k-1}$ into~\eqref{Sup:Eq:StochDRSvsIPA_b}~\cite{Ryu.etal2019}.
It is straightforward to verify that~\eqref{Sup:Eq:StochDRSvsIPA_a} can also be rewritten as
\begin{equation}
\label{Sup:Eq:StochDRS1}
\vbm^k = \vbm^{k-1} - \Ssf_{i_k}(\vbm^{k-1})\text{ \, with \, }\Ssf_{i_k} \defn \Dsf - \Gsf_{i_k}(2\Dsf-\Isf).
\end{equation}
\medskip
Then, for any $\vbmast \in \zer(\Ssf)$, we have that
\begin{align*}
&\|\vbm^k - \vbmast\|^2 \\
&=\|\vbm^{k-1} - \vbmast - \Ssf\vbm^{k-1}\|^2 + \|\Ssf\vbm^{k-1} - \Ssf_{i_k}\vbm^{k-1}\|^2 \\
& \quad + 2(\Ssf\vbm^{k-1} - \Ssf_{i_k}\vbm^{k-1})^\Tsf(\vbm^{k-1} - \vbmast - \Ssf\vbm^{k-1}) \\
&\leq \|\vbm^{k-1}-\vbmast\|^2 - \|\Ssf\vbm^{k-1}\|^2 + \|\Ssf\vbm^{k-1} - \Ssf_{i_k}\vbm^{k-1}\|^2\\
& \quad + 2 \|\Ssf\vbm^{k-1} - \Ssf_{i_k}\vbm^{k-1}\| \|\vbm^{k-1}-\vbmast\| \\
&\leq \|\vbm^{k-1}-\vbmast\|^2 - \|\Ssf\vbm^{k-1}\|^2 + \|\Ssf\vbm^{k-1} - \Ssf_{i_k}\vbm^{k-1}\|^2\\
& \quad + 2(R+2\gamma L) \|\Ssf\vbm^{k-1} - \Ssf_{i_k}\vbm^{k-1}\| \\
\end{align*}
where in the first inequality we used Cauchy-Schwarz  and~\eqref{Sup:Eq:MonRedIter}, and in the second inequality we used Lemma~\ref{Sup:Lem:IterBoun} in Appendix~\ref{Sup:Sec:LemmasThm1}.
By taking the conditional expectation on both sides, invoking Lemma~\ref{Sup:Lem:NoiseBound} in Appendix~\ref{Sup:Sec:LemmasThm1}, and rearranging the terms, we get
\begin{align*}
\|\Ssf\vbm^{k-1}\|^2 & \leq \|\vbm^{k-1}-\vbmast\|^2 - \E\left[\|\vbm^k-\vbmast\|^2 \;|\; \vbm^{k-1}\right]\\
& \quad  + 4\gamma L R + 12\gamma^2 L^2.
\end{align*}
Hence, by averaging over $t \geq 1$ iterations and taking the total expectation, we obtain
\begin{equation*}
\E\left[\frac{1}{t}\sum_{k = 1}^t \|\Ssf\vbm^{k-1}\|^2\right] \leq \frac{(R+2\gamma L)^2}{t} + 4 \gamma L R + 12\gamma^2 L^2.
\end{equation*}
The final result is obtained by noting that
\begin{equation*}
4\gamma L R + 12\gamma^2 L^2 \leq \max\{\gamma, \gamma^2\} (4LR + 12L^2).
\end{equation*}

\subsubsection{Lemmas Useful for the Proof of Theorem~~\ref{Thm:ConvThm1}}
\label{Sup:Sec:LemmasThm1}

This section presents two technical lemmas used in our analysis in Appendix~\ref{Sup:Sec:MainProof}.

\begin{lemma}
\label{Sup:Lem:NoiseBound}
Assume that Assumptions~\ref{As:WeaklyConvexData}-\ref{As:ExistenceOfFixed} hold and let $i_k$ be a uniform random variable over $\{1, \dots, b\}$. Then, we have that
\begin{equation*}
\E\left[\|\Ssf_{i_k}\vbm - \Ssf\vbm\|^2\right] \leq 4\gamma^2 L^2,\quad \vbm \in \R^n.
\end{equation*}
\end{lemma}
\begin{proof}\let\qed\relax
Let $\zbm_i = \Gsf_i(\xbm)$ and $\zbm = \Gsf(\xbm)$ for any $1 \leq i \leq b$ and $\xbm \in \R^n$.  From the optimality conditions for each proximal operator
\begin{equation*}
\Gsf_i\xbm = \prox_{\gamma g_i}(\xbm) = \xbm - \gamma\gbm_i(\zbm_i), \quad\gbm_i(\zbm_i) \in \partial g_i(\zbm_i)
\end{equation*}
and
\begin{equation*}
\Gsf\xbm = \prox_{\gamma g}(\xbm) = \xbm - \gamma\gbm(\zbm)
\end{equation*}
such that
\begin{equation*}
\gbm(\zbm) = \frac{1}{b}\sum_{i = 1}^b \gbm _i(\zbm) \in \partial g(\zbm),
\end{equation*}
where we used Proposition~\ref{Prop:SubdifferentialSum} in Supplement~\ref{Sec:NonsmoothOptimization}. By using the bound on all the subgradients (due to Assumption~\ref{As:WeaklyConvexData} and Proposition~\ref{Prop:SubdifferentialLipschitz} in Supplement~\ref{Sec:NonsmoothOptimization}), we obtain
\begin{align*}
\left\|\Gsf_i(\xbm) - \Gsf(\xbm)\right\|  & = \|\prox_{\gamma g_i}(\xbm)-\prox_{\gamma g}(\xbm)\| \\
& = \gamma\|\gbm_i(\zbm_i)-\gbm(\zbm)\| \leq 2\gamma L,
\end{align*}
where $L > 0$ is the Lipschitz constant of all $g_i$s and $g$. This inequality directly implies that
\begin{align*}
\|\Ssf\vbm-\Ssf_i\vbm\| = \|\Gsf(2\Dsf\vbm-\vbm)-\Gsf_i(2\Dsf\vbm-\vbm)\| \leq 2\gamma L.
\end{align*}
Since, this inequality holds for every $i$, it also holds in expectation.
\end{proof}

\begin{lemma}
\label{Sup:Lem:IterBoun}
Assume that Assumptions~\ref{As:WeaklyConvexData}-\ref{As:ExistenceOfFixed} hold and let the sequence $\{\vbm^k\}$ be generated via the iteration~\eqref{Sup:Eq:StochDRS1}. Then, for any $k \geq 1$, we have that
\begin{equation*}
\|\vbm^k-\vbmast\| \leq (R + 2\gamma L) \quad\text{for all}\quad \vbmast \in \zer(\Ssf).
\end{equation*}
\end{lemma}

\begin{proof}\let\qed\relax
The optimality of the proximal operator in~\eqref{Sup:Eq:StochDRS1} implies that there exists  $\gbm_{i_k}(\zbm^k) \in \partial g_{i_k}(\zbm^k)$ such that
\begin{align*}
& \zbm^k = \Gsf_{i_k}(2\xbm^{k-1}-\vbm^{k-1}) \\ \quad\Leftrightarrow\quad
& 2\xbm^{k-1} - \vbm^{k-1} - \zbm^k = \gamma \gbm_{i_k}(\zbm^k).
\end{align*}
By applying $\vbm^k =  \vbm^{k-1} - \Ssf_{i_k}(\vbm^{k-1}) = \vbm^{k-1} + \zbm^k - \xbm^{k-1}$ to the equality above, we obtain
\begin{equation*}
\xbm^{k-1} - \vbm^k = \gamma \gbm_{i_k}(\zbm^k) \quad\Leftrightarrow\quad \vbm^k = \xbm^{k-1} -  \gamma \gbm_{i_k}(\zbm^k).
\end{equation*}
Additionally, for any $\vbmast \in \zer(\Ssf)$ and $\xbmast = \Dsf(\vbmast)$, we have that
\begin{align*}
&\Ssf(\vbmast) = \Dsf(\vbmast) - \Gsf(2\Dsf(\vbmast)-\vbmast) = \xbmast - \Gsf(2\xbmast - \vbmast) = \zerobm\\
&\Rightarrow\quad \xbmast - \vbmast = \gamma \gbm(\xbmast) \quad\text{for some}\quad \gbm(\xbmast) \in \partial g(\xbmast).
\end{align*}
Thus, by using Assumption~\ref{As:ExistenceOfFixed} and the bounds on all the subgradients (due to Assumption~\ref{As:WeaklyConvexData} and Proposition~\ref{Prop:SubdifferentialLipschitz} in Supplement~\ref{Sec:NonsmoothOptimization}), we obtain
\begin{align*}
\|\vbm^k-\vbmast\|
&= \|\xbm^{k-1}-\gamma \gbm_{i_k}(\zbm^k) - \xbmast - \gamma \gbm(\xbmast)\| \\
&\leq \|\xbm^{t-1}-\xbmast\| + 2\gamma L \leq (R+2 \gamma L).
\end{align*}
\end{proof}

\subsection{Analysis of IPA for Strongly Convex Functions}
\label{Sec:StrongConvex}
In this section, we perform analysis of IPA under a different set of assumptions, namely under the assumptions adopted in~\cite{Ryu.etal2019}.

\medskip\noindent
\begin{assumption}
\label{Sup:As:StronglyConvexData}
Each $g_i$ is proper, closed, strongly convex with constant $M_i > 0$, and Lipschitz continuous with constant $L_i > 0$. We define the smallest strong convexity constant as $M = \min\{M_1, \dots, M_b\}$ and the largest Lipschitz constant as $L = \max\{L_1, \dots, L_b\}$.
\end{assumption}
This assumption further restricts Assumption~\ref{As:WeaklyConvexData} in the main paper to strongly convex functions.

\medskip\noindent
\begin{assumption}
\label{Sup:As:ContDen}
The residual $\Rsf_\sigma \defn \Isf - \Dsf_\sigma$ of the denoiser $\Dsf_\sigma$ is a contraction. It thus satisfies
\begin{equation*}
\|\Rsf\xbm - \Rsf\ybm\| \leq \epsilon \|\xbm -\ybm\|,
\end{equation*}
for all $\xbm, \ybm \in \R^n$ for some constant  $0 < \epsilon < 1$.
\end{assumption}
This assumption replaces Assumption~\ref{As:AveragedDenoiser} in the main paper by assuming that the residual of the denoiser is a contraction. Note that this can be practically imposed on deep neural net denoisers via spectral normalization~\cite{Miyato.etal2018}. We can then state the following.

\begin{theorem}
\label{Sup:Thm:ConvThm3}
Run IPA for $t \geq 1$ iterations with random i.i.d.~block selection under Assumptions~\ref{As:ExistenceOfFixed}-\ref{Sup:As:ContDen} using a fixed penalty parameter $\gamma > 0$. Then, the iterates of IPA satisfy
\begin{equation*}
\E\left[\|\xbm^k-\xbmast\|\right] \leq \eta^k (2R+4\gamma L) + \frac{4\gamma L}{1-\eta}, \quad 0 < \eta < 1.
\end{equation*}
\end{theorem}

\begin{proof}\let\qed\relax
It was shown in Theorem 2 of~\cite{Ryu.etal2019} that under Asumptions~\ref{Sup:As:StronglyConvexData} and~\ref{Sup:As:ContDen}, we have that
\begin{equation}
\label{Sup:Eq:Constant}
\|(\Isf-\Ssf)\xbm-(\Isf-\Ssf)\ybm\| \leq \eta \|\xbm - \ybm\|
\end{equation}
with
\begin{equation*}
\eta \defn \left(\frac{1+\epsilon + \epsilon \gamma M + 2\epsilon^2\gamma M}{1 + \gamma M + 2\epsilon \gamma M}\right),
\end{equation*}
for all $\xbm, \ybm \in \R^n$, where $\Ssf$ is given in~\eqref{Eq:DRSMonOperator}. Hence, when
\begin{equation*}
\frac{\epsilon}{\gamma M (1+\epsilon-2\epsilon^2)} < 1,
\end{equation*}
the operator $(\Isf-\Ssf)$ is a contraction.
Using the reasoning in Appendix~\ref{Sec:ProofThm1}, the sequence $\vbm^k = \zbm^k - \sbm^{k-1}$ can be written as
\begin{equation}
\label{Sup:Eq:StochDRS2}
\vbm^k = \vbm^{k-1} - \Ssf_{i_k}(\vbm^{k-1}) \text{ \,  with \, } \Ssf_{i_k} \defn \Dsf - \Gsf_{i_k}(2\Dsf-\Isf).
\end{equation}
\medskip
Then, for any $\vbmast \in \zer(\Ssf)$, we have that
\begin{align*}
&\|\vbm^k - \vbmast\|^2\\
&= \|(\Isf-\Ssf)\vbm^{k-1}-(\Isf-\Ssf)\vbmast\|^2 \\
& \quad + 2((\Isf-\Ssf)\vbm^{k-1}-(\Isf-\Ssf)\vbmast)^\Tsf((\Isf-\Ssf_{i_k})\vbm^{k-1}-\\
& \quad (\Isf-\Ssf)\vbm^{k-1}) + \|(\Isf-\Ssf_{i_k})\vbm^{k-1}-(\Isf-\Ssf)\vbm^{k-1}\|^2 \\
&\leq \eta^2\|\vbm^{k-1}-\vbmast\|^2 + 2\eta \|\vbm^{k-1}-\vbmast\| \|\Ssf_{i_k}\vbm^{k-1}-\Ssf\vbm^{k-1}\| \\
& \quad + \|\Ssf_{i_k}\vbm^{k-1}-\Ssf\vbm^{k-1}\|^2,
\end{align*}
where we used the Cauchy-Schwarz inequality and the fact that $(\Isf-\Ssf)$ is $\eta$-contractive. By taking the conditional expectation on both sides, invoking Lemma~\ref{Sup:Lem:NoiseBound} in Appendix~\ref{Sup:Sec:LemmasThm1}, and completing the square, we get
\begin{equation*}
\E\left[\|\vbm^k-\vbmast\|^2 | \vbm^{k-1}\right] \leq \left(\eta \|\vbm^{k-1}-\vbmast\|+ 2\gamma L\right)^2.
\end{equation*}
Then, by applying the Jensen inequality and taking the total expectation, we get
\begin{equation*}
\E\left[\|\vbm^k-\vbmast\|\right] \leq \eta \E\left[\|\vbm^{k-1}-\vbmast\|\right] + 2\gamma L.
\end{equation*}
By iterating this result and invoking Lemma~\ref{Sup:Lem:IterBoun} from Appendix~\ref{Sup:Sec:LemmasThm1}, we obtain
\begin{equation*}
\E\left[\|\vbm^k-\vbmast\|\right] \leq \eta^k (R+2\gamma L) + (2\gamma L)/(1-\eta).
\end{equation*}
Finally by using the nonexpansiveness of $(1/(1+\epsilon))\Dsf$ (see Lemma~9 in~\cite{Ryu.etal2019}) and the fact that $\xbmast = \Dsf(\vbmast)$, we obtain
\begin{align*}
\E\left[\|\xbm^k-\xbmast\|\right] & \leq (1+\epsilon) \left[\eta^k (R+2\gamma L) + \frac{2\gamma L}{1-\eta}\right]\\
& \leq \eta^k (2R+4\gamma L) + \frac{4\gamma L}{1-\eta}.
\end{align*}
This concludes the proof.
\end{proof}

\subsection{Fixed Point Interpretation}
\label{Sec:Equilibrium}

Fixed points of PnP algorithms have been extensively discussed in the recent literature~\cite{Meinhardt.etal2017, Buzzard.etal2017, Ryu.etal2019}. Our goal in this section is to revisit this topic in a way that leads to a more intuitive equilibrium interpretation of PnP. Our formulation has been inspired from the classical interpretation of ADMM as an algorithm for computing a zero of a sum of two monotone operators~\cite{Eckstein.Bertsekas1992}.

\subsubsection{Equilibrium Points of PnP Algorithms}

It is known that a fixed point $(\xbmast, \zbmast, \sbmast)$ of PnP-ADMM (and of all PnP algorithms~\cite{Meinhardt.etal2017}) satisfies
\begin{subequations}
\label{Sup:Eq:EquilCond1}
\begin{align}
&\xbmast = \Gsf(\xbmast + \sbmast)\\
&\xbmast = \Dsf(\xbmast - \sbmast),
\end{align}
\end{subequations}
with $\xbmast = \zbmast$, where $\Gsf = \prox_{\gamma g}$. Consider the \emph{inverse} of $\Dsf$ at $\xbm \in \R^n$, which is a set-valued operator $\Dsf^{-1}(\xbm) \defn \{\zbm \in \R^n : \xbm = \Dsf_\sigma(\zbm)\}$. Note that the inverse operator exists even when $\Dsf$ is not a bijection (see Section~2 of~\cite{Ryu.Boyd2016}). Then, from the definition of $\Dsf^{-1}$ and optimality conditions of the proximal operator, we can equivalently rewrite~\eqref{Sup:Eq:EquilCond1} as follows
\begin{align*}
&\sbmast \in \gamma \partial g(\xbmast) \text{\quad and \quad }
-\sbmast \in \Dsf^{-1}(\xbmast)-\xbmast.
\end{align*}
This directly leads to the following equivalent representation of PnP fixed points
\begin{equation}
\zerobm \in \Tsf(\xbmast) \defn \gamma \partial g(\xbmast) + (\Dsf^{-1}(\xbmast)-\xbmast).
\end{equation}
Hence, a vector $\xbmast$ computed by PnP can be interpreted as an equilibrium point between two terms with $\gamma > 0$ explicitly influencing the balance.

\subsubsection{Equivalence of Zeros of $\Tsf$ and $\Ssf$}
\label{Sup:Sec:EquivZer}

Define $\vbmast \defn \zbmast - \sbmast$ for a given fixed point $(\xbmast, \zbmast, \sbmast)$ of PnP-ADMM and consider the operator
\begin{equation*}
\Ssf = \Dsf - \Gsf(2\Dsf-\Isf) \quad\text{with}\quad \Gsf = \prox_{\gamma g},
\end{equation*}
which was defined in~\eqref{Eq:DRSMonOperator} of the main paper.
Note that from~\eqref{Sup:Eq:EquilCond1}, we also have $\xbmast = \Dsf(\vbmast)$ and $\vbmast = \xbmast - \sbmast$ (due to $\zbmast = \xbmast$). We then have the following equivalence
\begin{align*}
&\zerobm \in \Tsf(\xbmast) = \gamma \partial g(\xbmast) + (\Dsf^{-1}(\xbmast)-\xbmast) \\
&\Leftrightarrow\quad
\begin{cases}
\xbmast = \Gsf(\xbmast + \sbmast)\\
\xbmast = \Dsf(\xbmast - \sbmast)
\end{cases}\\
&\Leftrightarrow\quad
\begin{cases}
\xbmast = \Gsf(2\xbmast - \vbmast)\\
\xbmast = \Dsf(\vbmast)
\end{cases}\\
&\Leftrightarrow\quad
\Ssf(\vbmast) = \Dsf(\vbmast)-\Gsf(2\Dsf(\vbmast)-\vbmast) = \zerobm,
\end{align*}
where we used the optimality conditions of the proximal operator $\Gsf$. Hence, the condition that $\vbmast = \zbmast - \sbmast \in \zer(\Ssf)$ is equivalent to $\xbmast = \Dsf(\vbmast) \in \zer(\Tsf)$.

\subsubsection{Firm Nonexpansiveness of $\Ssf$}
\label{Sup:Eq:FirmNonExpS}

We finally would like to show that under Assumptions~\ref{As:WeaklyConvexData}-\ref{As:ExistenceOfFixed}, the operator $\Ssf$ is firmly nonexpansive. Assumption~\ref{As:AveragedDenoiser} and Proposition~\ref{Prop:ProxNonexpansive} in Supplement~\ref{Sec:NonsmoothOptimization} imply that $\Dsf$ and $\Gsf$ are firmly nonexpansive. Then, Proposition~\ref{Prop:NonexpEquiv} in Supplement~\ref{Sec:MonotoneOperators} implies that $(2\Dsf-\Isf)$ and $(2\Gsf-\Isf)$ are nonexpansive. Thus, the composition $(2\Gsf-\Isf)(2\Dsf-\Isf)$ is also nonexpansive and
\begin{equation}
\label{Sup:Eq:ADMMFixOp}
(\Isf - \Ssf) = \frac{1}{2}\Isf + \frac{1}{2}(2\Gsf-\Isf)(2\Dsf-\Isf)
\end{equation}
is $(1/2)$-averaged. Then, Proposition~\ref{Prop:NonexpEquiv} in Supplement~\ref{Sec:MonotoneOperators} implies that $\Ssf$ is firmly nonexpansive.

\subsection{Convergence Analysis of PnP-ADMM}
\label{Sec:PnPADMMAnalysis}

The following analysis has been adopted from~\cite{Ryu.etal2019}. For completeness, we summarize the key results useful for our own analysis by restating them under the assumptions in the main paper.

\subsubsection{Equivalence between PnP-ADMM and PnP-DRS}
\label{Sec:ADMMDRS}

An elegant analysis of PnP-ADMM emerges from its interpretation as the Douglas–Rachford splitting (DRS) algorithm~\cite{Ryu.etal2019}. This equivalence is well-known and has been extensively studied in the context of convex optimization~\cite{Eckstein.Bertsekas1992}. Here, we restate the relationship for completeness.

Consider the following DRS (top) and ADMM (bottom) sqeuences
\begin{align*}
&\begin{cases}
\xbm^{k-1} = \Dsf(\vbm^{k-1})\\
\zbm^k = \Gsf(2\xbm^{k-1}-\vbm^{k-1})\\
\vbm^k = \vbm^{k-1} + \zbm^k - \xbm^{k-1}
\end{cases}\\
\quad\Leftrightarrow\quad
&\begin{cases}
\zbm^k = \Gsf(\xbm^{k-1} + \sbm^{k-1})\\
\xbm^k = \Dsf(\zbm^k - \sbm^{k-1}) \\
\sbm^k = \sbm^{k-1} + \xbm^k - \zbm^k,
\end{cases}
\end{align*}
where $\Gsf \defn \prox_{\gamma g}$ is the proximal operator and $\Dsf$ is the denoiser. To see the equivalence between them, simply introduce the variable change $\vbm^k = \zbm^k - \sbm^{k-1}$ into DRS. Note also the DRS sequence can be equivalently written as
\begin{equation*}
\vbm^k = \vbm^{k-1} - \Ssf(\vbm^{k-1}) \quad\text{with}\quad \Ssf \defn \Dsf- \Gsf(2\Dsf-\Isf).
\end{equation*}
To see this simply rearrange the terms in DRS as follows
\begin{align*}
\vbm^k &= \vbm^{k-1} + \Gsf(2\xbm^{k-1}-\vbm^{k-1}) - \xbm^{k-1} \\
&= \vbm^{k-1} - \left[\Dsf(\vbm^{k-1})-\Gsf(2\Dsf(\vbm^{k-1})-\vbm^{k-1})\right].
\end{align*}

\subsubsection{Convergence Analysis of PnP-DRS and PnP-ADMM}

It was established in Appendix~\ref{Sup:Eq:FirmNonExpS} that $\Ssf$ defined in~\eqref{Eq:DRSMonOperator} of the main paper is firmly nonexpansive.

\medskip
Consider a single iteration of DRS $\vbm^+ = \vbm - \Ssf\vbm$. Then, for any $\vbmast \in \zer(\Ssf)$, we have
\begin{align*}
\|\vbm^+-\vbmast\|^2
&= \|\vbm-\vbmast\|^2 - 2(\Ssf\vbm-\Ssf\vbmast)^\Tsf(\vbm-\vbmast) + \|\Ssf\vbm\|^2\\
&\leq \|\vbm-\vbmast\|^2 - \|\Ssf\vbm\|^2,
\end{align*}
where we used $\Ssf\vbmast = \zerobm$ and firm nonexpansiveness of $\Ssf$. By rearranging the terms, we obtain the following upper bound at iteration $k \geq 1$
\begin{equation}
\label{Sup:Eq:UppBoundDRS}
\|\Ssf\vbm^{k-1}\|^2 \leq \|\vbm^{k-1}-\vbmast\|^2 - \|\vbm^k-\vbmast\|^2.
\end{equation}
By averaging the inequality~\eqref{Sup:Eq:UppBoundDRS} over $t \geq 1$ iterations, we obtain
\begin{equation*}
\frac{1}{t}\sum_{k = 1}^t \|\Ssf\vbm^{k-1}\|^2 \leq \frac{\|\vbm^0-\vbmast\|^2}{t} \leq \frac{(R+2\gamma L)^2}{t}
\end{equation*}
where used the bound on $\|\vbm^0-\vbmast\| \leq (R+2\gamma L)$ that can be easily obtained by following the steps in Lemma~\ref{Sup:Lem:IterBoun} in Appendix~\ref{Sup:Sec:LemmasThm1}.

This result directly implies that $\|\Ssf\vbm^t\| \rightarrow 0$ as $t \rightarrow 0$. Additionally, Krasnosel'skii-Mann theorem (see Section 5.2 in~\cite{Bauschke.Combettes2017}) implies that $\vbm^t \rightarrow \zer(\Ssf)$. Then, from continuity of $\Dsf$, we have that $\xbm^t = \Dsf(\vbm^t) \rightarrow \zer(\Tsf)$ (see also Appendix~\ref{Sup:Sec:EquivZer}). This completes the proof.

\bibliographystyle{IEEEtran}


\begin{figure*}[t]
	\begin{center}
		{\huge Supplementary Material for Scalable Plug-and-Play  ADMM with Convergence Guarantees}
	\end{center}
\end{figure*}
\newpage
\maketitle
\makeatletter

\subsection{Background material}
\label{Sec:BackgroundMaterial}

This section summarizes well-known results from the optimization literature that can be found in different forms in standard textbooks~\cite{Rockafellar.Wets1998, Boyd.Vandenberghe2004, Nesterov2004, Bauschke.Combettes2017}.

\subsubsection{Properties of Monotone Operators}
\label{Sec:MonotoneOperators}

\medskip
\begin{definition}
	An operator $\Tsf$ is Lipschitz continuous with constant $\lambda > 0$ if
	\begin{equation*}
	\|\Tsf\xbm - \Tsf\ybm\| \leq \lambda\|\xbm-\ybm\|,\quad \xbm, \ybm \in \R^n.
	\end{equation*}
	When $\lambda = 1$, we say that $\Tsf$ is nonexpansive. When $\lambda < 1$, we say that $\Tsf$ is a contraction.
\end{definition}

\medskip
\begin{definition}
	$\Tsf$ is monotone if
	\begin{equation*}
	(\Tsf\xbm-\Tsf\ybm)^\Tsf(\xbm-\ybm) \geq 0,\quad \xbm, \ybm \in \R^n.
	\end{equation*}
	We say that it is strongly monotone or coercive with parameter $\mu > 0$ if
	\begin{equation*}
	(\Tsf\xbm-\Tsf\ybm)^\Tsf(\xbm-\ybm) \geq \mu\|\xbm-\ybm\|^2,\quad \xbm, \ybm \in \R^n.
	\end{equation*}
\end{definition}

\begin{definition}
	$\Tsf$ is cocoercive with constant $\beta > 0$ if
	\begin{equation*}
	(\Tsf\xbm-\Tsf\ybm)^\Tsf(\xbm-\ybm) \geq \beta\|\Tsf\xbm-\Tsf\ybm\|^2, \quad \xbm, \ybm \in \R^n.
	\end{equation*}
	When $\beta = 1$, we say that $\Tsf$ is firmly nonexpansive.
\end{definition}

\medskip\noindent
The following results are derived from the definition above.

\begin{proposition}
	\label{Prop:NonexpCocoerOp}
	Consider $\Rsf = \Isf - \Tsf$ where $\Tsf: \R^n \rightarrow \R^n$.
	\begin{equation*}
	\Tsf \text{ is nonexpansive } \,\Leftrightarrow\, \Rsf \text{ is $(1/2)$-cocoercive.}
	\end{equation*}
\end{proposition}

\begin{proof}\let\qed\relax
	First suppose that $\Rsf$ is $1/2$ cocoercive. Let $\hbm \defn \xbm - \ybm$ for any $\xbm, \ybm \in \R^n$. We then have
	\begin{equation*}
	\frac{1}{2}\|\Rsf\xbm-\Rsf\ybm\|^2 \leq (\Rsf\xbm-\Rsf\ybm)^\Tsf\hbm = \|\hbm\|^2 - (\Tsf\xbm-\Tsf\ybm)^\Tsf\hbm.
	\end{equation*}
	We also have that
	\begin{equation*}
	\frac{1}{2}\|\Rsf\xbm-\Rsf\ybm\|^2 = \frac{1}{2}\|\hbm\|^2 - (\Tsf\xbm-\Tsf\ybm)^\Tsf\hbm + \frac{1}{2}\|\Tsf\xbm-\Tsf\ybm\|^2.
	\end{equation*}
	By combining these two and simplifying the expression
	\begin{equation*}
	\|\Tsf\xbm-\Tsf\ybm\| \leq \|\hbm\|.
	\end{equation*}
	The converse can be proved by following this logic in reverse.
\end{proof}

\begin{proposition}
	\label{Prop:ContStrongMon}
	Consider $\Rsf = \Isf - \Tsf$ where $\Tsf: \R^n \rightarrow \R^n$.
	\begin{align*}
	& \Tsf \text{ is Lipschitz continuous with constant } \lambda < 1 \\
	\Rightarrow \quad  & \Rsf \text{ is $(1-\lambda)$-strongly monotone.}
	\end{align*}
\end{proposition}

\begin{proof}\let\qed\relax
	By using the Cauchy-Schwarz inequality, we have for all $\xbm, \ybm \in \R^n$
	\begin{align*}
	& (\Rsf\xbm-\Rsf\ybm)^\Tsf(\xbm-\ybm) \\
	&= \|\xbm-\ybm\|^2 - (\Tsf\xbm-\Tsf\ybm)^\Tsf(\xbm-\ybm) \\
	&\geq \|\xbm-\ybm\|^2 - \|\Tsf\xbm-\Tsf\ybm\|\|\xbm-\ybm\| \\
	&\geq \|\xbm-\ybm\|^2 - \lambda\|\xbm-\ybm\|^2 \geq (1-\lambda)\|\xbm-\ybm\|^2.
	\end{align*}
\end{proof}

\begin{definition}
	For a constant $\alpha \in (0, 1)$, we say that $\Tsf$ is $\alpha$-averaged, if there exists a nonexpansive operator $\Nsf$ such that $\Tsf = (1-\alpha)\Isf + \alpha \Nsf$.
\end{definition}

The following characterization is often convenient.
\begin{proposition}
	\label{Prop:AveragedEquiv}
	For a nonexpansive operator $\Tsf$, a constant $\alpha \in (0, 1)$, and the operator ${\Rsf \defn \Isf-\Tsf}$, the following are equivalent
	\begin{enumerate}[label=(\alph*), leftmargin=*]
		\item $\Tsf$ is $\alpha$-averaged
		\item $(1-1/\alpha)\Isf + (1/\alpha)\Tsf$ is nonexpansive
		\item $\|\Tsf\xbm - \Tsf\ybm\|^2 \leq \|\xbm-\ybm\|^2 - \left(\frac{1-\alpha}{\alpha}\right)\|\Rsf\xbm-\Rsf\ybm\|^2,\, \xbm, \ybm \in \R^n$.
	\end{enumerate}
\end{proposition}

\begin{proof}\let\qed\relax
	See Proposition~4.35 in~\cite{Bauschke.Combettes2017}.
\end{proof}

\begin{proposition}
	\label{Prop:NonexpEquiv}
	Consider $\Tsf: \R^n \rightarrow \R^n$ and $\beta > 0$. Then, the following are equivalent
	\begin{enumerate}[label=(\alph*), leftmargin=*]
		\item $\Tsf$ is $\beta$-cocoercive
		\item $\beta\Tsf$ is firmly nonexpansive
		\item $\Isf-\beta\Tsf$ is firmly nonexpansive.
		\item $\beta\Tsf$ is $(1/2)$-averaged.
		\item $\Isf-2\beta\Tsf$ is nonexpansive.
	\end{enumerate}
\end{proposition}

\begin{proof}\let\qed\relax
	For any $\xbm, \ybm \in \R^n$, let $\hbm \defn \xbm-\ybm$. The equivalence between (a) and (b) is readily observed by defining $\Psf \defn \beta\Tsf$ and noting that
	\begin{align}\nonumber
	&(\Psf\xbm - \Psf\ybm)^\Tsf\hbm = \beta(\Tsf\xbm - \Tsf\ybm)^\Tsf\hbm \nonumber
	\quad\text{and}\quad \\
	\nonumber
	& \|\Psf\xbm-\Psf\ybm\|^2 = \beta^2 \|\Tsf\xbm-\Tsf\ybm\|.
	\end{align}
	
	\medskip\noindent
	Define $\Rsf \defn \Isf - \Psf$ and suppose (b) is true, then
	\begin{align*}
	&(\Rsf\xbm-\Rsf\ybm)^\Tsf\hbm\\
	& = \|\hbm\|^2 - (\Psf\xbm-\Psf\ybm)^\Tsf\hbm \\
	&= \|\Rsf\xbm-\Rsf\ybm\|^2 + (\Psf\xbm-\Psf\ybm)^\Tsf\hbm - \|\Psf\xbm-\Psf\ybm\|^2 \\
	&\geq \|\Rsf\xbm-\Rsf\ybm\|^2.
	\end{align*}
	By repeating the same argument for $\Psf = \Isf - \Rsf$, we establish the full equivalence between (b) and (c).
	
	\medskip\noindent
	The equivalence of (b) and (d) can be seen by noting that
	\begin{align*}
	&2\|\Psf\xbm-\Psf\ybm\|^2 \leq 2(\Psf\xbm-\Psf\ybm)^\Tsf\hbm \\
	\quad \Leftrightarrow\quad & \|\Psf\xbm-\Psf\ybm\|^2 \leq 2(\Psf\xbm-\Psf\ybm)^\Tsf\hbm - \|\Psf\xbm-\Psf\ybm\|^2 \\
	&= \|\hbm\|^2-(\|\hbm\|^2 - 2(\Psf\xbm-\Psf\ybm)^\Tsf\hbm + \|\Psf\xbm-\Psf\ybm\|^2)\nonumber\\
	&= \|\hbm\|^2 - \|\Rsf\xbm-\Rsf\ybm\|^2.
	\end{align*}
	
	\medskip\noindent
	To show the equivalence with (e), first suppose that ${\Nsf \defn \Isf - 2 \Psf}$ is  nonexpansive, then ${\Psf = \frac{1}{2}(\Isf + (-\Nsf))}$ is $1/2$-averaged, which means that it is firmly nonexpansive. On the other hand, if $\Psf$ is firmly nonexpansive, then it is $1/2$-averaged, which means that from Proposition~\ref{Prop:AveragedEquiv}(b) we have that $(1-2)\Isf + 2\Psf = 2\Psf - \Isf = -\Nsf$ is nonexpansive. This directly means that $\Nsf$ is nonexpansive.
\end{proof}

\subsubsection{Convex functions, subdifferentials, and proximal operators}
\label{Sec:NonsmoothOptimization}

\medskip\noindent
\begin{proposition}
	\label{Prop:MonotoneSubgradient}
	Let $f$ be a proper, closed, and convex function. Then for all $\xbm, \ybm \in \R^n$,  $\gbm \in \partial f(\xbm)$, and $ \hbm \in \partial f(\ybm)$,  $\partial f$  is a monotone operator
	\begin{equation*}
	(\gbm-\hbm)^\Tsf(\xbm-\ybm) \geq 0.
	\end{equation*}
	Additionally if $f$ is strongly convex with constant $\mu > 0$, then $\partial f$ is strongly monotone with the same constant.
	\begin{equation*}
	(\gbm-\hbm)^\Tsf(\xbm-\ybm) \geq \mu\|\xbm-\ybm\|^2.
	\end{equation*}
\end{proposition}

\begin{proof}\let\qed\relax
	Consider a strongly convex function $f$ with a constant $\mu \geq 0$. Then, we have that
	\begin{align*}
	& \begin{cases}
	f(\ybm) \geq f(\xbm) + \gbm^\Tsf(\ybm-\xbm) + \frac{\mu}{2}\|\ybm-\xbm\|^2 \\
	f(\xbm) \geq f(\ybm) + \hbm^\Tsf(\xbm-\ybm) + \frac{\mu}{2}\|\xbm-\ybm\|^2
	\end{cases}\\
	\quad\Rightarrow\quad &  (\gbm-\hbm)^\Tsf(\xbm-\ybm) \geq \mu \|\xbm-\ybm\|^2.
	\end{align*}
	The proof for a weakly convex $f$ is obtained by considering $\mu = 0$ in the inequalities above.
\end{proof}

\medskip\noindent
It is well-known that the proximal operator is firmly nonexpansive.
\begin{proposition}
	\label{Prop:ProxNonexpansive}
	Proximal operator $\prox_{\gamma f}$ of a proper, closed, and convex $f$ is firmly nonexpansive.
\end{proposition}

\begin{proof}\let\qed\relax\let\qed\relax
	Denote with $\xbm_1 = \Gsf\zbm_1 = \prox_{\gamma f}(\zbm_1)$ and $\xbm_2 = \Gsf\zbm_2 =\prox_{\gamma f}(\zbm_2)$, then
	\begin{align*}
	&\begin{cases}
	(\zbm_1 - \xbm_1) \in \gamma \partial f(\xbm_1) \\
	(\zbm_2 - \xbm_2) \in \gamma \partial f(\xbm_2)
	\end{cases}\\
	\Rightarrow\quad
	&(\zbm_1-\xbm_1-\zbm_2+\xbm_2)^\Tsf(\xbm_1-\xbm_2) \geq 0\\
	\Rightarrow\quad
	&(\Gsf\zbm_1-\Gsf\zbm_2)^\Tsf(\zbm_1-\zbm_2) \geq \|\Gsf\zbm_1-\Gsf\zbm_2\|^2
	\end{align*}
\end{proof}
\medskip\noindent
The following proposition is sometimes referred to as \emph{Moreau-Rockafellar theorem}. It establishes that for functions defined over all of $\R^n$, we have that $\partial f = \partial f_1 + \cdots + \partial f_m$.

\begin{proposition}
	\label{Prop:SubdifferentialSum}
	Consider ${f = f_1 + \cdots + f_m}$, where $f_1, \dots, f_m$ are proper, closed, and convex functions on $\R^n$.  Then
	\begin{equation*}
	\partial f_1(\xbm) + \cdots + \partial f_m(\xbm) \subset \partial f(\xbm), \quad \xbm \in \R^n
	\end{equation*}
	Moreover, suppose that convex sets $\ri(\dom f_i)$ have a point in common, then we also have
	\begin{equation*}
	\partial f(\xbm) \subset \partial f_1(\xbm) + \cdots + \partial f_m(\xbm), \quad \xbm \in \R^n.
	\end{equation*}
\end{proposition}

\begin{proof}\let\qed\relax
	See Theorem 23.8 in~\cite{Rockafellar1970a}.
\end{proof}

\medskip\noindent
\begin{proposition}
	\label{Prop:SubdifferentialLipschitz}
	Let $f$ be a convex function, then we have that
	\begin{align*}
	& f \text{ is Lipschitz continuous with constant $L > 0$} \\
	\quad\Leftrightarrow\quad & \|\gbm(\xbm)\| \leq L, \quad\gbm(\xbm) \in \partial f(\xbm), \quad\xbm \in \R^n.
	\end{align*}
\end{proposition}

\begin{proof}\let\qed\relax
	First assume that $\|\gbm(\xbm)\| \leq L$ for all subgradients. Then, from the definition of subgradient
	\begin{align*}
	& \begin{cases}
	f(\xbm) \geq f(\ybm) + \gbm(\ybm)^\Tsf(\xbm-\ybm) \\
	f(\ybm) \geq f(\xbm) + \gbm(\xbm)^\Tsf(\ybm-\xbm)
	\end{cases} \\
	\quad\Leftrightarrow\quad
	& \gbm(\ybm)^\Tsf(\xbm-\ybm) \leq f(\xbm) - f(\ybm) \leq  \gbm(\xbm)^\Tsf(\xbm-\ybm).
	\end{align*}
	Then, from Cauchy-Schwarz inequality, we obtain
	\begin{align*}
	-L \|\xbm-\ybm\| & \leq -\|\gbm(\ybm)\| \|\xbm-\ybm\| \\
	& \leq f(\xbm) - f(\ybm) \leq  \|\gbm(\xbm)\| \|\xbm-\ybm\| \leq L \|\xbm-\ybm\|.
	\end{align*}
	Now assume that $g$ is $L$-Lipschitz continuous. Then, we have for any $\xbm, \ybm \in \R^n$
	\begin{equation*}
	\gbm(\xbm)^\Tsf(\ybm-\xbm) \leq f(\ybm)-f(\xbm) \leq L\|\ybm-\xbm\|.
	\end{equation*}
	Consider $\vbm = \ybm - \xbm \neq \zerobm$, then we have that
	\begin{equation*}
	\gbm(\xbm)^\Tsf\left(\frac{\vbm}{\|\vbm\|}\right) \leq L.
	\end{equation*}
	Since, this must be true for any $\vbm \neq \zerobm$, we directly obtain
	$\|\gbm(\xbm)\| \leq L$.
\end{proof}

\begin{figure}[t]
	\centering\includegraphics[width=.9\linewidth]{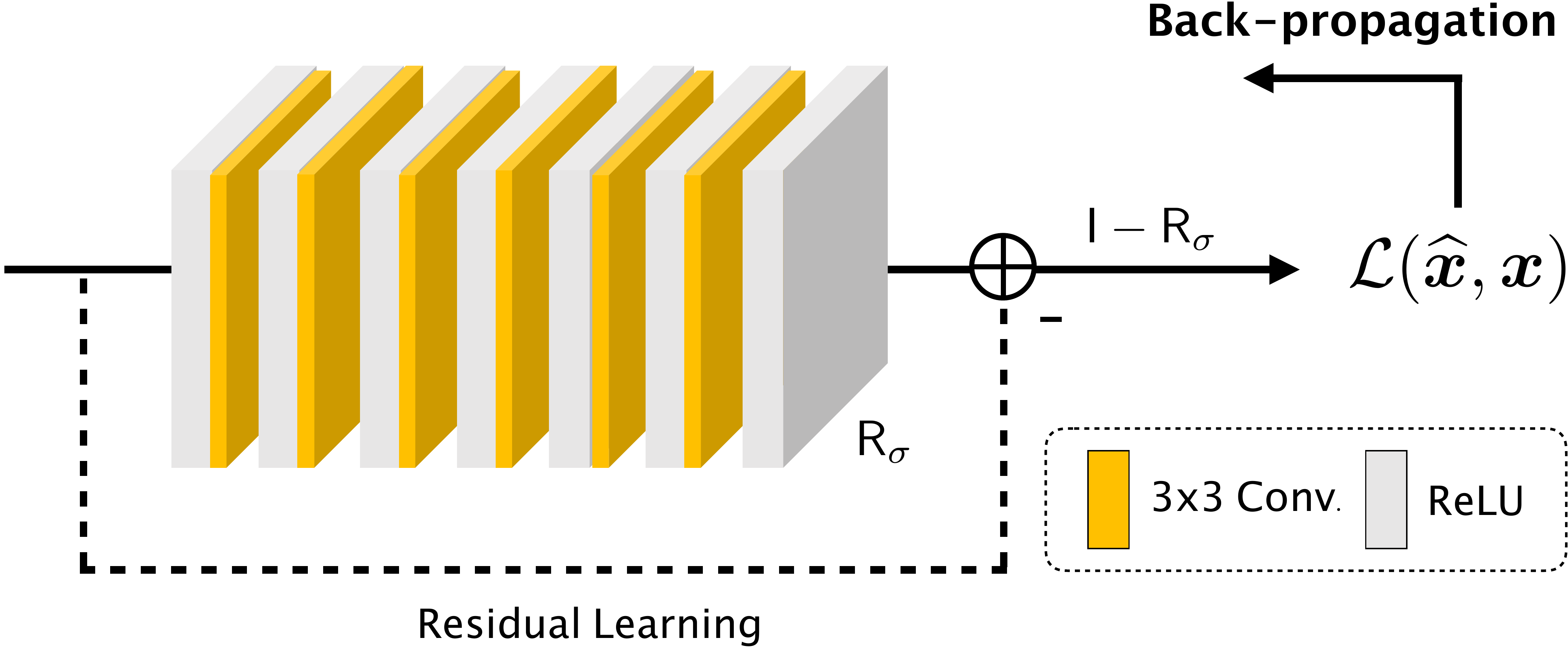}
	\caption{\emph{Illustration of the architecture of DnCNN used in all experiments. Vectors $\xbmhat$ and $\xbm$ denote the denoised image and ground truth, respectively. The neural net is trained to remove the AWGN from its noisy input image. We also constrains the Lipschitz constant of $\Rsf_\sigma$ to be smaller than $1$ by using the spectral normalization technique in~\cite{Sedghi.etal2019}. This provides a necessary condition for the satisfaction of Assumption~\ref{As:AveragedDenoiser}.}}
	\label{Fig:Architecture}
\end{figure}

\subsection{Additional Technical Details}
\label{Sec:TechnicalDetails}

In this section, we present several technical details that were omitted from the main paper due to length restrictions. Section~\ref{Sec:ArchitectureTraining} discusses the architecture and training of the DnCNN prior. Section~\ref{Sec:ExtraValidations} presents extra details and validations that compliment the experiments in the main paper with additional insights for IPA.

\renewcommand{\arraystretch}{1.2}
\begin{table*}[t]
	\centering
	\caption{Per-iteration memory usage specification for reconstructing 512$\times$512 images}
	\label{Tab:Memory2}
	\scriptsize
	\vspace{5pt}
	\begin{tabular*}{500pt}{C{40pt}C{40pt} C{55pt}C{50pt} C{55pt}C{50pt}C{55pt}C{50pt}}
		\toprule
		\multicolumn{2}{c}{\textbf{Algorithms}} & \multicolumn{2}{c}{\proposed~(60)} & \multicolumn{2}{c}{PnP-ADMM (300)} & \multicolumn{2}{c}{PnP-ADMM (600)} \\
		\cmidrule(lr){3-4} \cmidrule(lr){5-6} \cmidrule(lr){7-8}
		\multicolumn{2}{c}{\textbf{Variables}} & size & memory & size & memory & size & memory  \\
		\midrule
		\multirow{2}{*}{$\{\Abm_i\}$}   & real & $512\times512\times60$ & 0.23 GB & $512\times512\times300$ & 1.17 GB & $512\times512\times600$ & 2.34 GB  \\[0.5ex]
		\cdashline{3-8}
		& imaginary & $512\times512\times60$ & 0.23 GB & $512\times512\times300$ & 1.17 GB & $512\times512\times600$ & 2.34 GB \\[0.5ex]
		\cdashline{3-8}
		\multicolumn{2}{c}{$\{\ybm_i\}$} & $512\times512\times60$  & 0.47 GB & $512\times512\times300$  & 2.34 GB & $512\times512\times600$ & 4.69 GB \\[0.5ex]
		\cdashline{3-8}
		\multicolumn{2}{c}{others combined} & --- & 0.03 GB & --- & 0.03 GB & --- & 0.03 GB\\
		\midrule
		\multicolumn{2}{c}{\textbf{Total}} & & \colorbox{lightgreen}{\textbf{0.97 GB}} & & \textbf{4.72 GB} & & \textbf{9.41 GB} \\
		\bottomrule
	\end{tabular*}
\end{table*}

\begin{figure*}[t]
	\centering\includegraphics[width=\linewidth]{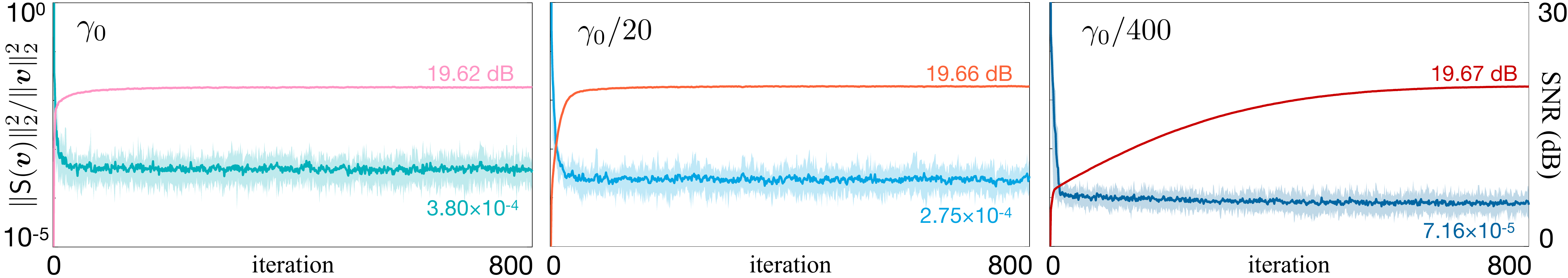}
	\caption{\emph{Illustration of the convergence of IPA for a DnCNN prior under drastically changed $\gamma$ values. The average normalized distance to $\zer(\Ssf)$ and SNR (dB) are plotted against the iteration number with the shaded areas representing the range of values attained over $12$ test images. In practice, the convergence speed improves with larger values of $\gamma$. However, \proposed~still can achieve same level of SNR results for a wide range of $\gamma$ values.}}
	\label{Fig:l2_convergence}
\end{figure*}

\begin{table*}[t]
	\centering
	\caption{Optimized SNR (dB) obtained by \proposed~under different priors for images from \emph{Set 12}}
	\label{Tab:SNR}
	\scriptsize
	\vspace{5pt}
	\begin{tabular*}{500pt}{C{70pt}C{70pt}C{70pt}C{70pt}C{70pt}C{70pt}}
		\toprule
		\multirow{2}{*}{\textbf{Algorithms}} & PnP-ADMM & \multicolumn{3}{c}{\proposed~(\textbf{Ours})} & PnP-ADMM \\
		& (Fixed 5) & \multicolumn{3}{c}{(Random 5 from full 60)} & (Full 60) \\
		\cmidrule(lr){2-2} \cmidrule(lr){3-5} \cmidrule(lr){6-6}
		\textbf{Denoisers} & DnCNN  & TV & BM3D & DnCNN & DnCNN   \\
		\midrule
		\textit{Cameraman} & 15.95     & 17.45     & 17.38     & 18.16       & 18.34         \\
		\textit{House}    & 19.22     & 21.79     & 21.97     & 22.45       & 22.94         \\
		\textit{Pepper}   & 17.06   & 18.68   & 19.55     & 20.60     & 21.11       \\
		\textit{Starfish} & 18.20    & 19.29     & 20.29     & 21.64     & 22.22       \\
		\textit{Monarch}  & 17.70    & 19.81     & 18.66     & 20.85     & 21.60           \\
		\textit{Aircraft}  & 17.15   & 18.67     & 18.83     & 19.28     & 19.54         \\
		\textit{Parrot}  & 17.13    & 18.60     & 18.27     & 18.72     & 19.18             \\
		\textit{Lenna}  & 15.41    & 16.48     & 16.32     & 16.94     & 17.13              \\
		\textit{Barbara}  & 13.63    & 16.00     & 17.53     & 16.58     & 16.85              \\
		\textit{Boat}  & 17.98    & 19.35     & 20.21     & 20.95     & 21.34              \\
		\textit{Pirate}  & 17.93    & 19.36     & 19.45     & 19.88     & 20.10              \\
		\textit{Couple}  & 15.40    & 17.31     & 17.53     & 18.24     & 18.57              \\
		\midrule
		\textbf{Average} & \textbf{16.90} & \textbf{18.57} & \textbf{18.83} & \colorbox{lightgreen}{\textbf{19.52}} & \textbf{19.91}        \\
		\bottomrule
	\end{tabular*}
\end{table*}
\begin{figure*}[t]
	\centering\includegraphics[width=0.95\linewidth]{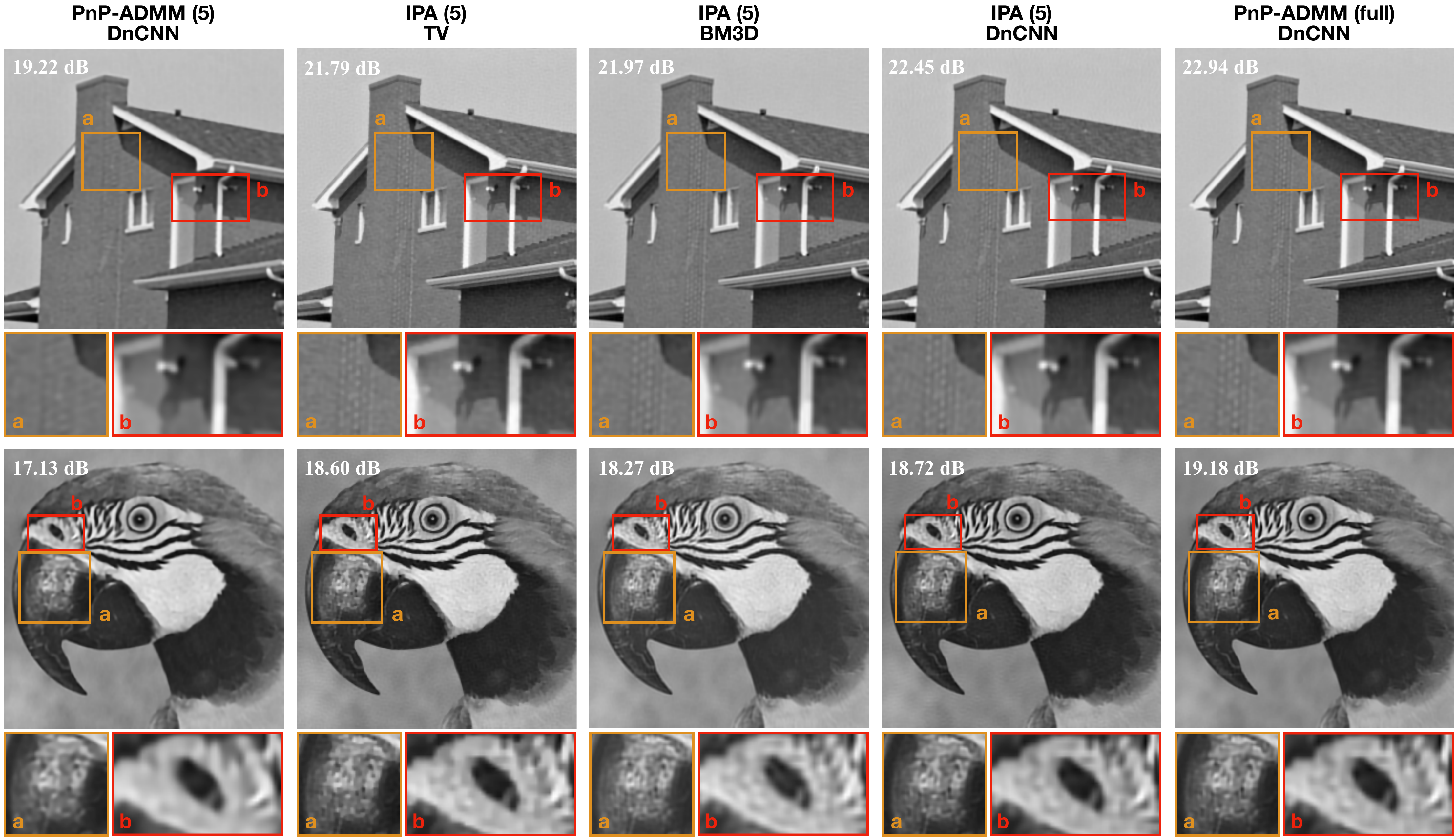}
	\caption{\emph{Visual examples of the reconstructed House (upper) and Parrot (bottom) images by \proposed~and PnP-ADMM. The first and last columns correspond to PnP-ADMM under DnCNN with 5 fixed measurements and with the full 60 measurements, respectively. The second, third, and fourth column correspond to \proposed~with a small minibatch of size 5 under TV, BM3D, and DnCNN, respectively. Each image is labeled by its SNR (dB) with respect to the original image, and the visual difference is highlighted by the boxes underneath. Note that \proposed~recovers the details lost by the batch algorithm with the same computational cost and achieves the same high-quality results as the full batch algorithm.}}
	\label{Fig:VisualExamples}
\end{figure*}

\subsubsection{Architecture and Training of the DnCNN Prior}
\label{Sec:ArchitectureTraining}

Fig.~\ref{Fig:Architecture} visualizes the architectural details of the DnCNN prior used in our experiments. In total, the network contains $7$ layers, of which the first $6$ layers consist of a convolutional layer and a rectified linear unit (ReLU), while the last layer is just a convolution. A skip connection from the input to the output is implemented to enforce residual learning. The output images of the first $6$ layers have $64$ feature maps while that of the last layer is a single-channel image. We set all convolutional kernels to be $3\times3$ with stride $1$, which indicates that intermediate images have the same spatial size as the input image. We generated 11101 training examples by adding AWGN to 400 images from the BSD400 dataset~\cite{Martin.etal2001} and extracting patches of  $128 \times 128$ pixels with stride $64$. We trained DnCNN to optimize the \emph{mean squared error} by using the Adam optimizer~\cite{Kingma.Ba2015}.

We use the spectral normalization technique in~\cite{Sedghi.etal2019} to control the global Lipschitz constant (LC) of DnCNN. In the training, we constrain the residual network $\Rsf_\sigma$ to have LC smaller than $1$. Since the firmly non-expansiveness implies non-expansiveness, this provides a \emph{necessary} condition for $\Rsf_\sigma$ to satisfy Assumption~\ref{As:AveragedDenoiser}.

\subsubsection{Extra Details and Validations for Optical Tomography}
\label{Sec:ExtraValidations}

All experiments are run on the machine equipped with an Intel Core i7 Processor that has 6 cores of 3.2 GHz and 32 GBs of DDR memory. We trained all neural nets using NVIDIA RTX 2080 GPUs. We define the SNR (dB) used in the experiments as
\begin{equation}
\operatorname{SNR}(\hat{\xbm}, \xbm) \triangleq \max _{a, b \in \mathbb{R}}\left\{20 \operatorname{log} _{10}\left(\frac{\|\xbm\|_{\ell_{2}}}{\|\xbm-a \hat{\xbm}+b\|_{\ell_{2}}}\right)\right\}, \nonumber
\end{equation}
where $\hat{\xbm}$ represents the estimate and $\xbm$ denotes the ground truth.

For intensity diffraction tomography, we implemented an epoch-based selection rule due to the large size of data. We randomly divide the measurements (along with the corresponding forward operators) into non-overlapping chunks of size $60$ and save these chunks on the hard drive. At every iteration, \proposed~loads only a single random chunk into the memory while the full-batch PnP-ADMM loads all chunks sequentially and process the full set of  measurements. This leads to the lower per iteration cost and less memory usage of \proposed~than PnP-ADMM. Table~\ref{Tab:Memory2} shows extra examples of the memory usage specification for reconstructing $512\times512$ pixel permittivity images. These results follow the same trend observed in Table~\ref{Tab:Memory} of the main paper.
We also conduct some extra validations that provides additional insights into \proposed. In these simulations, we use images of size $254\times254$ pixels from \emph{Set 12} as test examples. We assume real permittivity functions with the total number of measurement $b=60$.

Fig.~\ref{Fig:l2_convergence} illustrates the evolution of the convergence of IPA for different values of the penalty parameter. We consider three different values of $\gamma\in\{\gamma_0, \gamma_0/20, \gamma/400\}$ with $\gamma_0=20$. The average normalized distance $\|\Ssf(\vbm^k)\|_2^2/\|\vbm^k\|_2^2$ and SNR are plotted against the iteration number and labeled with their respective final values. The shaded areas represent the range of values attained across all test images. \proposed~randomly select $5$ measurements in every iteration to impose the data-consistency.
Fig.~\ref{Fig:l2_convergence} compliments the results in Fig~\ref{fig:l1_convergence} of the main paper by showing the fast convergence speed in practice with larger values of $\gamma$. On the other hand, this plot further demonstrates that \proposed~is stable in terms of the SNR results for a wide range of $\gamma$ values.

Prior work has discussed the influence of the denoising prior on the final result. Our last simulation compares the final reconstructed images of \proposed~by using TV, BM3D, and DnCNN. Since TV is a proximal operator, it serves as a baseline. Table~\ref{Tab:SNR} compares the average SNR values obtained by different image priors. We include the results of PnP-ADMM using $5$ fixed measurements and the full batch as reference. Visual examples of \emph{House} and \emph{Parrot} are shown in Fig.~\ref{Fig:VisualExamples}. First, the table numerically illustrates significant improvement of \proposed~over PnP-ADMM under the same computational budget. Second, leveraging learned priors in \proposed~leads to the better reconstruction than other priors. For instance, DnCNN outperforms TV and BM3D by $0.7$ dB in SNR. Last, the agreement between \proposed~and the full batch PnP-ADMM highlights the nearly optimal performance of our algorithm at a significantly lower computational cost and memory usage.

\end{document}